%% file: lidglm_arxiv.tex
\documentclass[10pt, a4paper]{article}

\usepackage{times}
\usepackage{soul}
\usepackage{url}
\usepackage[hidelinks]{hyperref}
\usepackage[utf8]{inputenc}
\usepackage[small]{caption}
\usepackage{graphicx}
\usepackage{amsmath}
\usepackage{amssymb}
\usepackage{amsthm}
\usepackage{authblk}
\usepackage{doi}
\usepackage{hyperref}
\usepackage{subcaption}
\usepackage{booktabs}
\usepackage{multirow}
\usepackage{thm-restate}
\usepackage{dsfont}
\usepackage{aligned-overset}
\usepackage{enumerate}
\usepackage[toc, page]{appendix}
\usepackage{xcolor}
\usepackage{csvsimple}
\usepackage[font = small]{caption}
\usepackage[]{natbib}
\usepackage{bookmark}

\graphicspath{{./illustrations/}{../ldglm/application_and_trials/paper_scripts/}}

\newtheorem{theo}{Theorem}[section]

\newtheorem{defi}[theo]{Definition}
\newtheorem{cor}[theo]{Corollary}
\newtheorem{lma}[theo]{Lemma}

\begin{document}
\title{\bf LiD-GLM: Lipschitz-constrained Deep Generalized Linear Models}
\author[1]{Tom A. Splittgerber\thanks{
    We thank Jens Behrmann for useful feedback on the project direction and manuscript.
    This work has been funded by the Deutsche Forschungsgemeinschaft (DFG, German Research Foundation) - 459360854 as part of the Research Unit 
    "Lifespan AI: From Longitudinal Data to Lifespan Inference in Health" (DFG FOR 5347), University of Bremen (https://lifespanai.de/).}}
    \author[2]{Niklas Koenen} \author[2]{Marvin N. Wright} \author[1]{Werner Brannath}
    \affil[1]{Competence Center for Clinical Trials Bremen, University of Bremen}
    \affil[2]{Leibniz Institute for Prevention Research and Epidemiology – BIPS \& University of Bremen}
  \maketitle

\bigskip
\begin{abstract}
The combination of traditional statistical models and neural network (NN) components into semi-structured hybrid models is an intriguing approach to construct models that, ideally, combine traditional interpretability with the unprecedented flexibility of NNs. In order to preserve interpretability, it is usually necessary to restrict the NN components to prevent them from dominating the model. However, existing methods that enforce structural constraints on their NN components severely limit their models' flexibility; in contrast, methods that only enforce weak, indirect constraints lose meaningful interpretability.
The method we propose therefore leverages invertible residual neural networks (i-ResNets) to equip generalized linear models with both nonlinear parameter estimation and a flexible correction of their distributional assumptions while always retaining stochastic monotonicity of the modeled distribution in the (formerly linear) predictor. The i-ResNets correspond to a controlled deviation from identity and by constraining their Lipschitz constant one can rigorously limit and quantify how far the hybrid model deviates from its traditional counterpart. This enables a user-specifiable compromise between flexibility and interpretability without limiting the structure of nonlinear and interaction effects that can be learned.
Furthermore, we develop specific inherent interpretation techniques for our model and enforce model identifiability through an adapted post-hoc orthogonalization.
\end{abstract}

\noindent%
\vfill

\include{tex_components/main_body}

\bibliography{updated_references.bib}


\appendix

\include{tex_components/supplement}

\end{document}

%% file: tex_components/main_body.tex
\section{Introduction}\label{sec: intro}
The complexity of modern machine learning (ML) models is often drastically higher than that of traditional statistical ones.
This gives them the expressive capacity
to, in many practical scenarios, achieve unprecedented predictive performance. 

As is well established, the price of this flexibility is a
lack of \emph{inherent interpretability}. 
At the current state of ML research, the interpretation of expressive models can usually only be approached with post-hoc methods \citep{molnar_interpretable_2025}. The degree to which such methods describe or recover the true inner workings of a complex model is often unclear, as they rely on assumptions and simplifications \citep{rudin_stop_2019}. 
As a result, there has been increasing interest in models which aim to be both expressive and inherently interpretable \citep{agarwal_neural_2021, ji_comprehensive_2025, rugamer_deepregression_2023}. Still, these models often have to enforce strict structural constraints or accept a significant loss in interpretability.

In this paper, we propose a conditional distribution model that emphasizes inherent interpretability and aims at making the trade-off between its expressiveness and interpretability explicit and controllable. Modeling the conditional distribution (and not only mean and variance) of a target random variable $Y$ given a vector of covariates $\textbf{X}=(X_1,\dots, X_k)$ also enables a precise statistical description of aleatoric uncertainty. 

We base our method on the widely used generalized linear models (GLMs) \citep{mccullagh_generalized_2019}, which are well interpretable and compatible with a wide variety of target data types (continuous, binary, categorical, etc.). 
More precisely, we assume for $Y$ a one-parameter exponential family, i.e.\ a density of the form
\begin{equation}\label{eq: glm-expFamily}
	f_Y(y;\theta, \phi)=\exp\left(\frac{y\theta-b(\theta)}{\phi}+c(y,\phi)\right),
\end{equation}
where $b(\cdot)$ and $c(\cdot,\cdot)$ are known functions 
and the \emph{canonical parameter} $\theta\in \mathbb{R}$ parameterizes the distribution. 
The covariate independent and positive \emph{dispersion parameter} $\phi$ is either known (e.g. $\phi=1$ for the Binomial distribution) or needs to be estimated from the data.
We make the typical regularity assumptions, e.g.\ twice differentiability of $b$ \citep{efron_exponential_2022,mccullagh_generalized_2019}.

The influence of the covariates on the canonical parameter $\theta$, and thereby on the (conditional) distribution of $Y$, is modelled via a known 
function $\theta=\Lambda(\eta)$ of a linear predictor $\eta= \beta_0 + \sum_{i=1}^{k} \beta_i x_i$. Typically, this function is defined as a composition of two inverse functions, namely the link function $\lambda$, that maps the conditional expectation $\mu$ of $Y$ to $\eta$, and the derivative function $b'(\theta)$, that maps the canonical parameter to $\mu$:
\begin{equation}\label{eq: glm-muEquation}
\theta=(b')^{-1}\cdot\lambda^{-1}(\eta)=\Lambda(\eta)=\Lambda\left(\beta_0 + \sum_{i=1}^{k} \beta_i x_i\right). 
\end{equation}
As the focus of this work is on the full distribution, rather than just the conditional mean, we work with $\eta$ and the function $\Lambda$. We also assume w.l.o.g. that $\lambda$ (and therefore also $\Lambda$) is strictly increasing (otherwise set $\widetilde{\lambda}=-\lambda$ or $\widetilde{\beta_i}=-\beta_i $ for all $i$).

The simple structure of a GLM makes it well suited for interpretation since $Y$ is stochastically increasing in $\theta$ (e.g. \citep{rink_confidence_2025}) and by Equation~\eqref{eq: glm-muEquation} also in $\eta=\beta_0 + \sum_{i=1}^{k} \beta_i x_i$. Therefore, the sign and size of the coefficient $\beta_i$ determines the direction and strength of the effect of each covariate $x_i$ on the distribution of $Y$.

Our approach aims to largely preserve this interpretability while addressing the two main issues of the GLM, namely the limited expressiveness of Equation~\eqref{eq: glm-muEquation} and the distributional restriction stemming from assumption~\eqref{eq: glm-expFamily}. 

While the expressiveness of \eqref{eq: glm-muEquation} will be carefully extended via a nonlinear transformation of the input variables, the distributional assumption~\eqref{eq: glm-expFamily} can be relaxed utilizing a general latent variable representation of parametric families whose existence is guaranteed by the following Theorem:
\begin{restatable}{theo}{trafoTheorem}
\label{thm: nu2-inverseTrafoThm}
	Let $\{f_\theta : \theta\in\mathcal{I}\}$ be a family of probability distributions on $\mathbb{R}$ parameterized by $\theta$. There then exist an absolutely continuous probability distribution with cumulative distribution function $G_0$ and a function $h:\mathcal{I}\times\mathbb{R}\to\mathbb{R}, (\theta,v)\mapsto y := h_\theta(v)$ such that if $\theta\in\mathcal{I}$ is arbitrary, $Y\sim f_\theta$ and $V\sim G_0$ then:
	\begin{equation}\label{eq: nu2-inverseTrafoThm}
		Y \overset{d}{=} h_\theta(V),
	\end{equation}
	where $\overset{d}{=}$ denotes equality in distribution. Furthermore, $h$ can be chosen to be increasing in both arguments.
\end{restatable}
The following Corollary provides a natural choice of $G_0$ for continuous exponential families. The proofs of Theorem~\ref{thm: nu2-inverseTrafoThm} and the Corollary can be found in Supplement~\ref{append: inverseTrafoProof}. 
\begin{cor}\label{cor: nu2-inverseTrafoThm-referenceDist}
    If $\{f_\theta : \theta\in\mathcal{I} \}$ is a family of continuous random variables, whose support is a single interval, we can choose $G_0=F_{\theta_0}$ for a fixed $\theta_0\in\mathcal{I}$.
\end{cor}

As examples recall that any exponentially distributed $Y\sim \text{Exp}(\theta)$ is obtained by the continuous transformation $Y=V/\theta$ of the standard exponential distribution $V\sim \text{Exp}(1)$, and a discretely distributed $Y\sim F_\theta$ with support $I\subseteq \mathbb{N}$ has the same distribution as $h_\theta(V)$ where 
$V$ follows the uniform distribution on $[0,1]$ and $h_\theta(v)=\min\{k\in I: v\leq F_\theta(k)\}$.     

In our model \eqref{eq: intro-T2} below, we apply normalizing flows to the latent variable $V$ of Theorem~\ref{thm: nu2-inverseTrafoThm}. In particular, this provides an elegant way to relax the distributional assumptions of discrete exponential families (like the binomial distribution) for which a direct monotone transformation of the target variable $Y$ is impossible without altering its support. 

\subsection{Contribution}\label{sec: intr-contrib}

We propose a semi-structured extension of GLMs that integrates an invertible residual neural network (i-ResNet) \citep{behrmann_invertible_2019, chen_residual_2019} into Equation~\eqref{eq: glm-muEquation} which allows a flexible but careful nonlinear extension of predictor $\eta$. We also relax assumption~\eqref{eq: glm-expFamily} using methods from the field of normalizing flows \citep{kobyzev_normalizing_2021}, utilizing a second i-ResNet applied to a latent variable representation as described in Theorem~\ref{thm: nu2-inverseTrafoThm}. 
We name the resulting model \emph{\textbf{Li}pschitz-constrained \textbf{D}eep \textbf{G}eneralized \textbf{L}inear \textbf{M}odel} (LiD-GLM):

\noindent\textbf{LiD-GLM Model:}
{\sl Our model starts with a given GLM with exponential family \eqref{eq: glm-expFamily} and regression parameters $\beta\in\mathbb{R}^k, \beta_0\in\mathbb{R}$ for the linear predictor $\eta\in\mathbb{R}$. We also assume a given representation of the corresponding exponential family via a latent variable $V$ and transformations $\{h_\theta, \theta\in\mathbb{R}\}$ (see Theorem~\ref{thm: nu2-inverseTrafoThm}).
    Let further $T_p=\text{Id}+\nu_p:\mathbb{R}^k\to\mathbb{R}^k$ and 
    $T_d=\text{Id}+\nu_d:\mathbb{R}\to\mathbb{R}$ be i-ResNets (see Supplement \ref{sec: append-iResNet-Decomp}) where ``Id'' is the respective identity function and denote by $T_{p,i}$ and $\nu_{p,i}$ the $i$-th components of the output vectors of $T_p$ respectively $\nu_p$. 
	Our model then assumes that, given a realization of covariates $\textbf{X}=\textbf{x}$: 
	\begin{equation}\label{eq: intro-T1}
        \tag{M1}
		\eta = \beta_0 + T_p(\textbf{x}) \cdot \beta  =
        \beta_0 + \sum_{i=1}^k\beta_i  T_{p,i}(\textbf{x})  =
           \beta_0 + \sum_{i=1}^k\beta_i \big(x_i+\nu_{p,i}(\textbf{x})\big)    ,
	\end{equation}
	(where ``$\cdot$'' denotes the scalar product) and:
	\begin{equation}\label{eq: intro-T2}
        \tag{M2}
		Y \sim  h_\theta(T_d(V)),
	\end{equation}
    which we illustrate in Figure \ref{fig: flowchComp}.
}

\begin{figure}
	\centering
\includegraphics[width=0.8\textwidth]{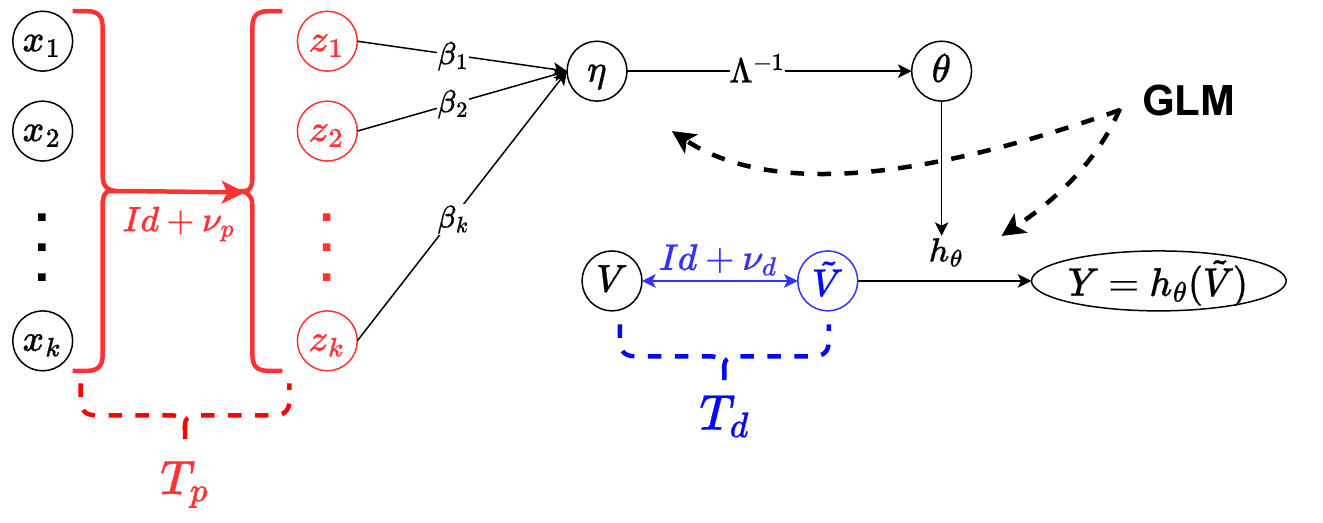}
		\caption{Flowchart of LiD-GLM. The extensions from Section~\ref{sec: nu1} are pictured in red and those of Section~\ref{sec: nu2} in blue. For $\nu_p\equiv 0$ and $\nu_d\equiv 0$, the structure of a GLM (black) is obtained.}
		\label{fig: flowExtGlm}
	\label{fig: flowchComp}
\end{figure}

Depending on the application, one can deliberately choose one of the transformations $T_p$ or $T_d$
as the identity function, thus enforcing the base GLM's assumptions on the corresponding model component. An example is the Bernoulli distribution which is fully determined by $\theta$ and thereby the transformation $T_d$ is obsolete.

An i-ResNet $T:\mathbb{R}^m\to \mathbb{R}^m$ is a concatenation of ``residual blocks'', which are invertible functions of the form $(\text{Id}+\nu):\mathbb{R}^m\to\mathbb{R}^m$, where Id is the $m$-dimensional identity function and $\nu$ a neural network (NN) with a constrained Lipschitz constant $L(\nu)\leq c, \; c<1$ (see Def. \ref{def: lip} below). These Lipschitz constraints allow a researcher to limit how far the extensions \eqref{eq: intro-T1} and \eqref{eq: intro-T2} deviate from \eqref{eq: glm-muEquation} and \eqref{eq: glm-expFamily}. An example is shown in Figure \ref{fig: nu_1-2dex}. Unlike similar methods, LiD-GLM does not require a pre-defined structure of nonlinearities or interactions.

\begin{figure}
    \centering
    \includegraphics[width=\linewidth]{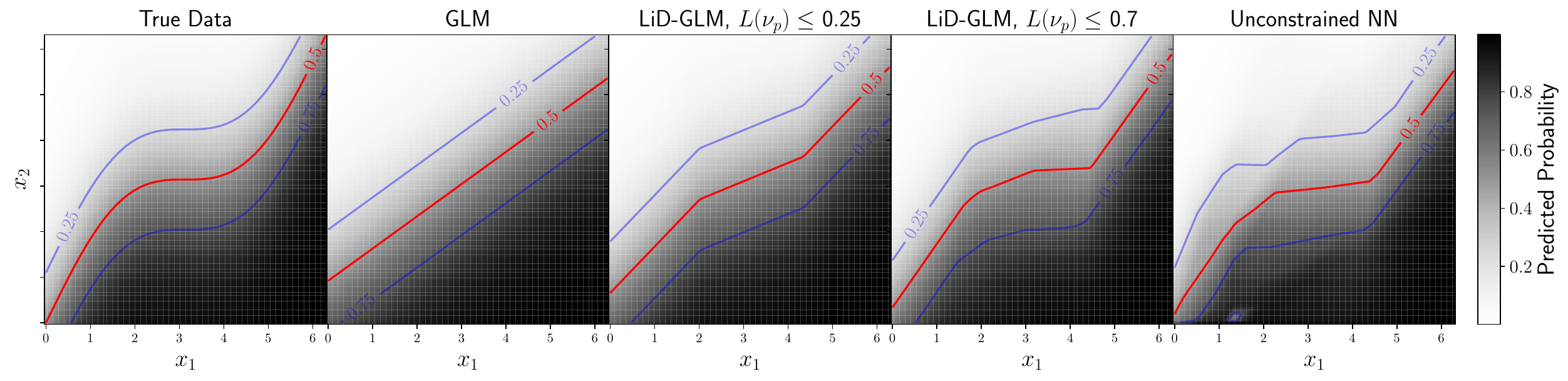}
	\caption{
    Heatmaps and contour lines ($\mathbb{P}=\frac{1}{4}, \frac{1}{2}, \frac{3}{4}$) showing learned densities in synthetic binary data example. Shown are 
    1) the true distribution $\mathbb{P}[Y=1|x_1,x_2]=\text{sigmoid}(x_1+\sin(x_1)-x_2)$, 2) logistic regression, 3+4) LiD-GLM with different bounds for $L(\nu_p)$ and 5) an unconstrained NN.}
	\label{fig: nu_1-2dex}
\end{figure}

The goal of the LiD-GLM model is to preserve the baseline of interpretability of the GLM while benefiting from the i-ResNets' flexibility. 
In particular, in \eqref{eq: intro-T2} the i-ResNet $T_d$ is invertible and $h_\theta$ is increasing in $\theta$ (see Theorem \ref{thm: nu2-inverseTrafoThm}), and therefore the stochastic monotonicity of $Y$ in $\eta$ remains a core model property, resulting in an intuitive and interpretable model behaviour. Also, the term $T_{p,i}(\textbf{x})=x_i+\nu_{p,i}(\textbf{x})$ remains to be monotone in $x_i$. However, in general, the monotonicity of $\eta$ in $x_i$ cannot be preserved, if the model is to learn complex interactions, which we discuss further in Supplement~\ref{sec: append-Monoton}.

To achieve identifiability, we adapt the ``PHO''-orthogonalization from \citep{rugamer_new_2023} and empirically verify its effectiveness through applications to real and synthetic data.
Finally, we develop interpretational methods for LiD-GLM and discuss these for applications on datasets of different complexity, comparing also model performance to comparable models.

\subsection{Related Literature}\label{sec: intro-relLit}

The concept of combining statistical parametric conditional distribution models with NNs has already been applied in existing literature for different reasons. In \citep{lim_deeply-learned_2023} and \citep{tran_bayesian_2020}, Deep-GLM models are proposed that contain a pre-transformation of a GLM's inputs by an unrestricted NN. 
In wide \& deep learning \citep{cheng_wide_2016}, a sum of predictors from a linear model and an unrestricted NN is proposed in the context of recommender systems. However, 
the interpretability of these unconstrained methods is limited, as no guarantee can be given to what degree the learned linear weights actually describe the models' behaviour.

Other models that focus more on interpretability have been constructed on the basis of generalized additive models (GAMs). In neural additive models (NAMs) \citep{agarwal_neural_2021}, the predictor can be computed from one-dimensional NN-functions of single input covariates. They can learn flexible nonlinearities while remaining relatively well interpretable, but are inflexible in learning interactions. In semi-structured deep distributional regression models (SDDR) \citep{rugamer_semi-structured_2024,rugamer_deepregression_2023}, linear predictors are combined with structured nonlinear predictors (e.g. splines) and NN predictors, each working on subsets of the input covariates. Also, SDDR models are based on generalized linear additive models for location, scale and shape (GAMLSS). However, the interpretability of SDDR models is strongly dependent on their assumed structure. Only SDDR models with few pre-specified interaction terms are interpretable, but may miss important but a priori overlooked interactions. Later we use an adapted version of an orthogonalization method developed in the context of SDDR \citep{rugamer_new_2023} and discuss the relation between SDDR and LiD-GLM in more detail in Supplement~\ref{sec: append-asSDDR}. 

All above mentioned methods
rely on low-dimensional parametric distributional assumptions. The high-dimensional distributional correction $T_d$ we propose in Equation~\eqref{eq: intro-T2} is instead closely related to conditional normalizing flows \citep{winkler_learning_2023} (see Supplement~\ref{sec: append-CondNormFlow}). The latter learn diffeomorphisms from an unknown true data distribution to a simple known distribution using invertible NNs. While normalizing flows (without any structural component) are not inherently interpretable,
the structure gained from using GLMs (and their latent space representation) in conjunction with $T_d$ results in inherent interpretability.

It has been previously shown that enforcing Lipschitz constraints on NNs can improve their robustness \citep{cisse_parseval_2017} and generalizability as well as improve performance in a \emph{small data} context \citep{gouk_regularisation_2021}, which is what our model is built for. Like \citep{behrmann_invertible_2019}, we use a fast but conservative \citep{anil_sorting_2019} layerwise upper bound on the NNs' Lipschitz constants during training; however, for a tighter bound during interpretation, we can make use of more recent exact computation methods \citep{splittgerber_exlipbab_2026} that are computationally feasible due to LiD-GLMs' relatively small size.

\subsection{Paper Outline}

The rest of this paper is structured as follows: in Section~\ref{sec: ldglm}, we will detail the proposed extensions of the traditional GLM with details on \eqref{eq: intro-T1} in Section~\ref{sec: nu1}
and details on Equation \eqref{eq: intro-T2} in Section~\ref{sec: nu2}. In Section~\ref{sec: ldglm-orthog}, we discuss a method to achieve identifiability via orthogonalization and in Section~\ref{sec: interpret} we discuss interpretation methods for LiD-GLM. In Section~\ref{sec: applications}, we then apply LiD-GLM to several datasets, comparing it to other methods, to illustrate its strengths and limitations. Finally, in Section~\ref{sec: discussion}, we set our proposed method into a broader statistical and machine learning context.

\section{Lipschitz-constrained DeepGLM (LiD-GLM)}\label{sec: ldglm}

In this section, we describe the extensions we make going from a GLM (with arbitrary family and link) to a LiD-GLM while using the notation of Section~\ref{sec: intro}.

\subsection{(i-)ResNet Input Transformation}\label{sec: nu1}
To extend \eqref{eq: glm-muEquation} to nonlinear predictions, we prepend to the GLM an i-ResNet $T_p:\mathbb{R}^k\to\mathbb{R}^k, \textbf{x}\mapsto \textbf{z}:=T_p(\textbf{x})$ operating on the input variables resulting in a GLM component which operates on a moderately ``pre-transformed'' covariate vector $\textbf{z}=T_p(\textbf{x})$:
\begin{equation}\label{eq: ldglm-nu1}
    \eta=\beta_0 + T_p(\textbf{x})\cdot \beta = \beta_0 + \big(\textbf{x} + \nu_p(\textbf{x})\big)\cdot\beta
    =\underbrace{\beta_0 + \textbf{x}\cdot\beta}_{\text{Linear predictor}}+\underbrace{\nu_p(\textbf{x})\cdot\beta}_{\text{NN term}}.
\end{equation}
This first extension to the GLM from \eqref{eq: intro-T1} is pictured on the left side of Figure~\ref{fig: flowExtGlm} and enables the extended model to learn nonlinear and interaction effects in its prediction of $\eta$. 

To guarantee a certain bound to which the linear term $\textbf{x} \cdot \beta$ dominates the predictions of the fitted LiD-GLM, we utilize the concept of Lipschitz continuity:
\begin{defi}[Lipschitz continuity]\label{def: lip}
	Given two normed spaces $(\mathcal{X},||.||_\mathcal{X}), (\mathcal{Y}, ||.||_\mathcal{Y})$, a function $f:\mathcal{X}\rightarrow \mathcal{Y}$ is called Lipschitz continuous if there exists a constant $L$ such that
	\begin{equation}
		||f(x_1)-f(x_2)||_\mathcal{Y}\leq L \cdot ||x_1-x_2||_\mathcal{X} \quad \forall x_1,x_2\in \mathcal{X}.
	\end{equation}
	With a slight abuse of nomenclature, we will refer to the \emph{Lipschitz norm}:
	\begin{equation}
		L(f):=\inf\{L\geq 0: f \text{ is } L \text{-Lipschitz}\}.
	\end{equation}
	of a Lipschitz continuous function $f$ as ``the'' Lipschitz constant of $f$, as it is shown in \citep{cobzas_lipschitz_2019} that $L(f)$ is the minimal Lipschitz constant for $f$.
\end{defi}

In an i-ResNet, the Lipschitz constant of the nonlinear term $\nu_p$ can be restricted by a user defined 
upper bound $L^b_p$. By definition, $L^b_p$ bounds the Jacobian norm of $\nu_p$ and therefore also the absolute value of the partial derivatives of $\nu_{p,i}$.
A bound $L^b_p=0$ in particular implies that $\nu_p$ equals a constant $c$ and thereby $T_p=Id+c$ and the predictor  
\eqref{eq: intro-T1} of the LiD-GLM model is linear as in a classical GLM. Moreover, the orthogonalization method from Section~\ref{sec: ldglm-orthog} below guarantees that $c=0$ in this case.
Also, if $L^b_p<1$ then the partial derivative of $T_{p,i}(\textbf{x})$ in $x_i$ is guaranteed to be positive which, as shown in 
Supplement \ref{sec: append-Monoton}, preserves the monotonicity of the transformed variable $T_{p,i}(\textbf{x})$ in $x_i$.

The bound $L^b_p$ is a controllable hyperparameter which limits the impact of $\nu_p$ and thereby the ``nonlinearity'' of \eqref{eq: intro-T1}. It can be deliberately chosen for individual practical applications. 
Noticeably, the bound $L^b_p$ does not directly limit the size of $\nu_p(\textbf{x})$ (which also depends on the support of $\textbf{x}$) but how much it changes with
changes in $\textbf{x}$ and thereby its deviation from a constant. 
The importance of $L^b_p$ is illustrated in Figure~\ref{fig: nu_1-2dex} and is further studied in Section~\ref{sec: applications}.

\subsubsection{Binary LiD-GLM Models}

For a binary target variable $Y$ the LiD-GLM model simplifies to a model for $p=\mathbb{P}(Y=1)$ of the form $p=\lambda^{-1}(\beta_0 + T_p(\textbf{x})\beta)$ where $\lambda$ is the link function of the GLM, e.g.\
the logit-function of a logistic model. As mentioned in the Section~\ref{sec: intr-contrib}, no additional distributional transformation is required and the specification of a binary LiD-GLM model is complete without the methods from the following Section~\ref{sec: nu2}.

\subsection{Distributional Correction}\label{sec: nu2}

We showed in Theorem~\ref{thm: nu2-inverseTrafoThm} that the distribution family assumption of a GLM \eqref{eq: glm-expFamily} can be rewritten as $Y\overset{d}{=}h_\theta(V)$ for a parametrized function $h_\theta$ and a continuous latent variable $V$ with a known distribution.
As the second main component of LiD-GLM, we proposed the inclusion of an i-ResNet $T_d$, i.e. $Y \sim  h_\theta(T_d(V))$ in Equation~\eqref{eq: intro-T2}, which corresponds to a $\theta$-independent correction of the latent variable distribution through a normalizing flow.

The effect of this transformation is illustrated in Figure \ref{fig: skewEx} for toy data.
We show in Supplement~\ref{append: inverseTrafoProof} that an $h_\theta$ can always be found such that \ref{eq: intro-T2} is compatible with backpropagation and can be used to generate new samples and calculate the density after training.

\begin{figure}[h]
\centering
\includegraphics[width=0.5\textwidth]{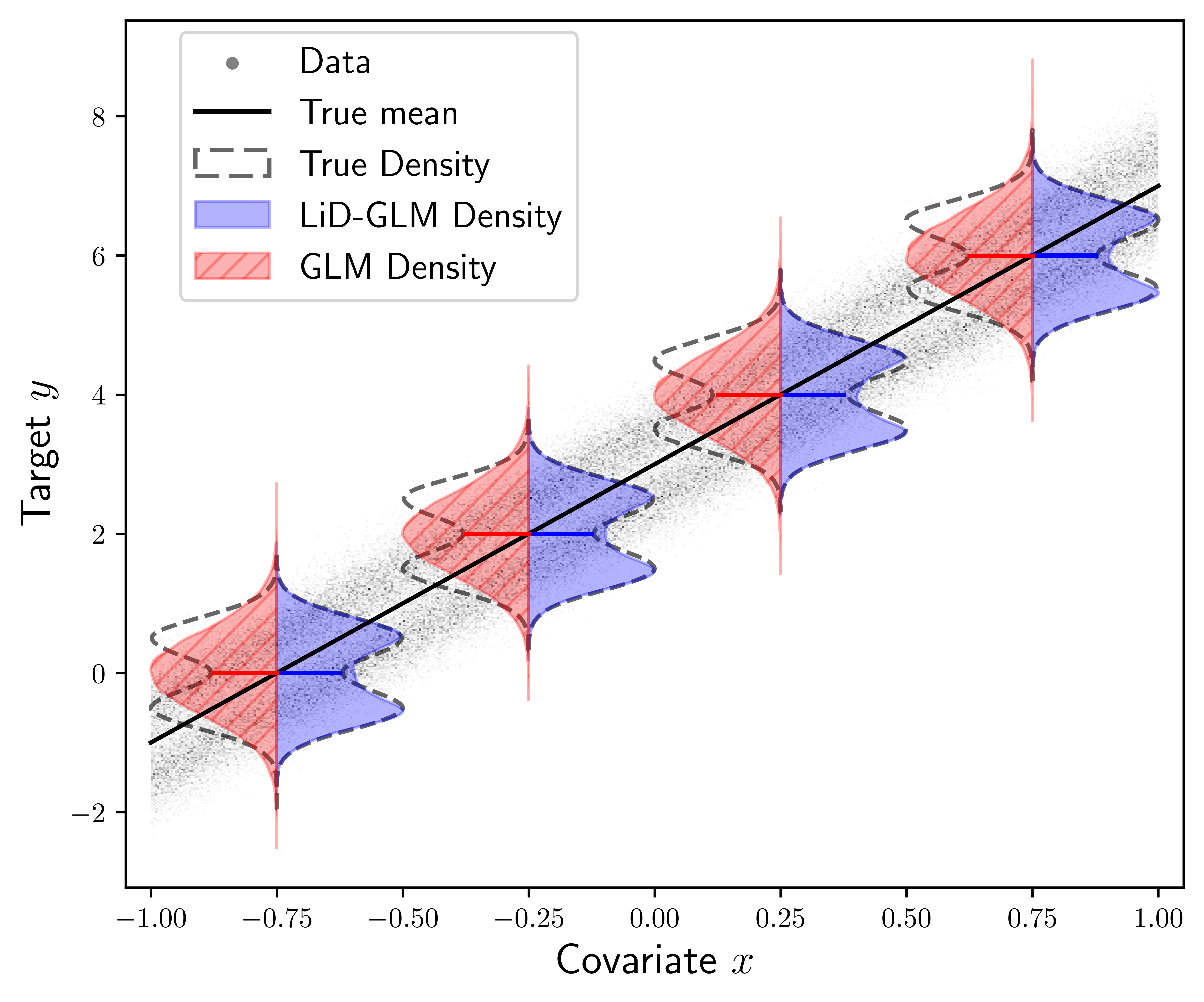}

	\caption{Example of distributional correction on synthetic multimodal data. Both GLM and LiD-GLM are able to learn the true conditional mean but the LiD-GLM also learns the approximate shape of the true conditional distribution which the GLM cannot.}

	\label{fig: skewEx}
\end{figure}

\subsubsection{Special Case: Normal Distributed Target}\label{sec: nu2-LM}

If we start with a normal linear regression model $Y=\eta+\epsilon$ we can use $V=\epsilon \sim\mathcal{N}(0, \sigma^2)$ and, because in this case $\theta=\eta$ holds, we can write $h_\eta(v) = v+\eta$ (the inverse of the centering function) to obtain $Y\overset{d}{=}h_\eta(V)$.  
The distributional correction from \eqref{eq: intro-T2} then becomes:
\begin{equation}\label{eq: nu2-LM-Td}
	Y = \eta + T_d(V),
\end{equation}
which corresponds to a correction of the original residual distribution through $T_d$.
Again, the bound $L_d^b$ on $L(T_d)$ governs the strength of this correction, which is illustrated in Figure \ref{fig: skewedEx2}. For $L_d^b=0$, the model remains with a normally distributed target.

As $T_d$ is mean-independent, \eqref{eq: intro-T2} still assumes a mean-independent residual distribution as was previously mentioned. This assumption could be weakened by using a higher-dimensional i-ResNet $T_d:\mathbb{R}^2\to\mathbb{R}^2,(\epsilon,\eta)\to (\tilde{\epsilon}, \eta)$.

\begin{figure}
	\centering
    \includegraphics[width=\textwidth]{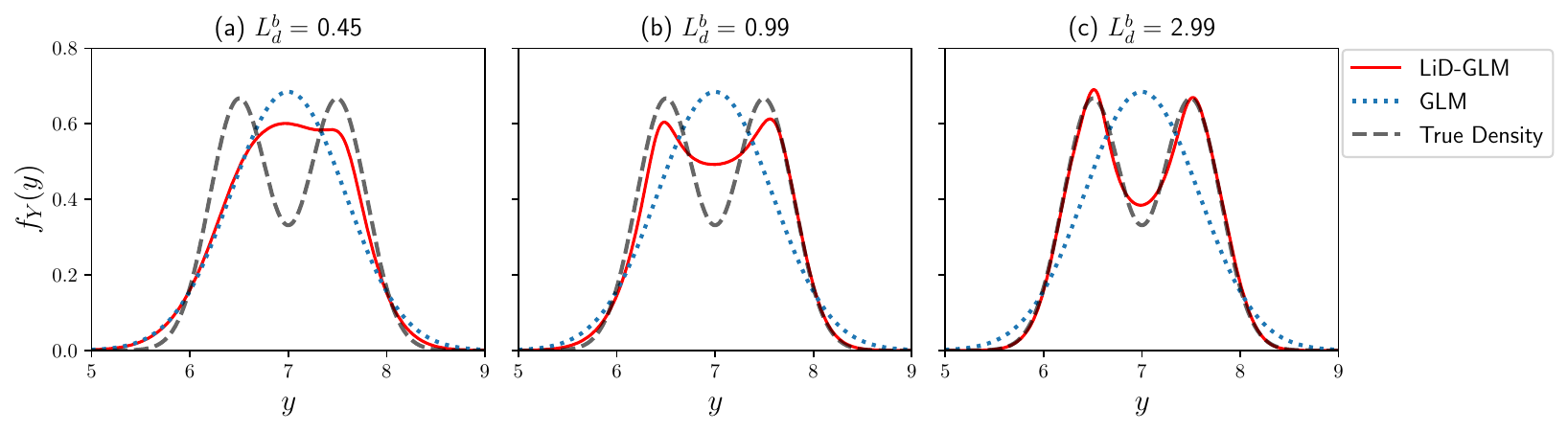}
	
	\caption{ Distributional correction learned by LiD-GLM for different $L^b_d$ on data as in Figure~\ref{fig: skewEx} shown for a fixed covariate $x=1$. 
    For increasing $L_d^b$, the conditional density of the LiD-GLM deviates more strongly from the base GLM and approximates the true density increasingly well.}
	\label{fig: skewedEx2}
\end{figure}

\subsection{Orthogonalization, Identifiability}\label{sec: ldglm-orthog}
If \eqref{eq: intro-T1} is rewritten as $\eta = \beta_0 + \textbf{x}\cdot\beta+\nu_p(\textbf{x})\cdot\beta$ as in \eqref{eq: ldglm-nu1}, an issue of identifiability becomes apparent: any non-trivial $\nu_p$ can absorb or cancel out the linear effects learned in $\beta$. This issue and its detrimental effect on model interpretability are discussed for SDDR models in \citep{rugamer_deepregression_2023,rugamer_semi-structured_2024,rugamer_new_2023} and we can build on their results.

It should however also be mentioned that LiD-GLM have some inherent protection against issues of identifiability as the Lipschitz constraint placed on $\nu_p$ and the usage of $\beta$ in both the linear and nonlinear terms prevent a free shifting of linear effects into $\nu_p$.
As we also initialize $\beta$ with the weights of a trained GLM and $T_p$ at the identity function, the gradient descent algorithm has to actively shift weight into the nonlinear component. We empirically show that both properties reduce issues of identifiability in Section~\ref{sec: appl-orth}.

However, to achieve full identifiability, we adapt post hoc orthogonalization (PHO) proposed in \citep{rugamer_new_2023} to component \eqref{eq: intro-T1} of LiD-GLM. As this requires a lengthy computation, we refer to Supplement~\ref{append: proofOrthog} for details. The main idea is to compute the projection of the output of $\nu_p$ into the column space of the training data $\textbf{X}$. This projection identifies the part of $\nu_p$ which can be described as a linear function of $\textbf{X}$ and $\beta$ can be adjusted accordingly. Also, the projection is subtracted from $\nu_p$. In effect, this method results again in a model of the form $\eta = \widetilde{\beta_0} + (\textbf{x} + \widetilde{\nu_p}(\textbf{x}))\widetilde{\beta}$ with the same predictions as the original model but with the property that if $\textbf{x}_i$ and $\widetilde{\nu}_{p, j}(\textbf{X})$ are arbitrary columns of \textbf{X} respectively $\widetilde{\nu_p}(\textbf{X})$, then $\textbf{x}_i \cdot \widetilde{\nu}_{p, j}(\textbf{X})= 0$, i.e.\ the columns are orthogonal. Each $\widetilde{\nu}_{p, j}(\textbf{X})$ is also orthogonal to the constant column $(1, \dots,1)^T$.

\section{Interpretation Methods}\label{sec: interpret}
In this section, we describe methods and guidelines for interpreting our model, describing model-specific methods and going beyond generic post-hoc interpretability methods such as partial dependence plots.
As a prerequisite, we apply the orthogonalization method from Section~\ref{sec: ldglm-orthog}, but for a simpler notation we still use the symbols $(\beta_0,\beta, \nu_p)$ instead of $(\widetilde{\beta_0}, \widetilde{\beta}, \widetilde{\nu_p})$ in this section.

\subsection{Interpreting $T_p$}\label{sec: interpret-Tp}
In the interpretation of a traditional GLM, the effect of any covariate $x_i$ on $\eta$ is given by $\beta_i$, independently of all other covariates. In our model \eqref{eq: intro-T1}, this relationship holds between the transformed covariates $z_i= T_{p,i}(\textbf{x})=x_i + \nu_{p,i}(\textbf{x})$ and $\eta$. The relationship between $x_i$ and $\eta$ is now more complex and not necessarily monotonous. However, the common coefficient $\beta_i$ for the linear and nonlinear term implies that both terms act on a common scale, and the smaller the Lipschitz constant of the nonlinear term, the more $\beta_i$ alone describes the influence of $x_i$ on the target variable.

As mentioned before, in the important special case $L(\nu_p)<1$ we can at least guarantee that each transformed covariate $z_i$ is strictly monotonically increasing in $x_i$. This even follows already if the Lipschitz norm of the $i$-th component $\nu_{p,i}(\textbf{x})$ is below $1$. The latter is bounded by the Lipschitz norm of the vector-wise function $\nu_p$ but can be smaller. The exact value by which it is bounded also depends on the chosen vector norms. Note that to obtain tight bounds, both the multidimensional and unidimensional Lipschitz constants can be computed post-hoc with exact computation methods like \citep{bhowmick_lipbab_2021} or \citep{splittgerber_exlipbab_2026}.

\subsubsection{Coefficient of Determination $R^2_i$}\label{sec: interpret-Tp-R2}

In this section, we present an alternative and more natural way to quantify how much the transformed variable $z_i=x_i+\nu_{p,i}(\textbf{x})$ is dominated by $x_i$. This then also quantifies how far $\beta_i$
still describes the effect of $x_i$ on $Y$. In this section, we denote by $\textbf{z}_i$ and $\textbf{x}_i \in \mathbb{R}^{n\times1}$ the $i$-th column of the transformed and untransformed design matrix respectively and by $\nu_{p,i}(\textbf{X})$ that of the linear correction term, meaning $\textbf{z}_i  = \textbf{x}_i + \nu_{p,i}(\textbf{X})$. Because $\nu_p$ was orthogonalized w.r.t.\ the constant one-vector $\iota_n=(1, \dots,1)^T$ during orthogonalization, the mean of each column is zero: $\overline{\nu_{p,i}(\textbf{X)}}=\iota_n^T \nu_{p,i}(\textbf{X)} =0$.
Therefore the means $\overline{\textbf{z}_i}  = \overline{\textbf{x}_i }$ are equal.

Using this, one can decompose the total variance of $\textbf{z}_i$, computed via the well-known sum of squares $||\textbf{z}_i - \overline{\textbf{z}_i}||_2^2$ as follows:
\begin{align*}
    ||\textbf{z}_i - \overline{\textbf{z}_i}\iota_n||_2^2 \quad \overset{\overline{\textbf{z}_i}  = \overline{\textbf{x}_i }}&{=} \quad ||\textbf{x}_i + \nu_{p,i}(\textbf{X}) - \overline{\textbf{x}_i}\iota_n||_2^2 
    \quad \overset{\nu_{p,i}\perp \iota_n, \textbf{x}_i}{=} \quad ||\textbf{x}_i - \overline{\textbf{x}_i}\iota_n||_2^2 + ||\nu_{p,i}(\textbf{X})||_2^2
\end{align*}
The coefficient of determination is then given as the fraction variance explained by the term $\textbf{x}_i$:
\begin{equation}
    R_i^2 := \frac{||\textbf{x}_i - \overline{\textbf{x}_i}\iota_n||_2^2}{||\textbf{z}_i - \overline{\textbf{z}_i}\iota_n||_2^2} .
\end{equation}
One can also see that 
because of PHO, $\textbf{x}_i$ is the ordinary-least-squares prediction from the linear regression $\textbf{z}_i\sim\textbf{x}_i$. This ties our $R_i^2$ back to its traditional definition.

\subsection{Interpreting $T_d$}\label{sec: interpr-T_d}
The i-ResNet $T_d$ is an invertible one-dimensional real-valued function and one can easily see that it is strictly increasing. Because the function $h_\theta$ from \eqref{eq: intro-T2} is also increasing in both arguments, it preserves this monotonicity in the latent variable which makes an interpretation of $T_d$ relatively straightforward.
In particular, plotting $T_d$ provides important insights into the fitted LiD-GLM, e.g.\ when compared to a diagonal line representing the identity. If the distribution $G_0$ for the latent variable is symmetric, an asymmetric deviation from the identity will indicate a transformation towards a skewed distribution, see Figures~\ref{fig: interpret-TdEx} (a.1, a.2).  We can also see whether an initial unimodal distribution $G_0$
(e.g.\ a Gaussian distribution) is transformed into a multimodal distribution, which can indicate a heterogeneity that is not covered by the covariates. An example is shown in Figure~\ref{fig: interpret-TdEx} (b.1, b.2).

\begin{figure}
	\centering
	\includegraphics[width=\textwidth]{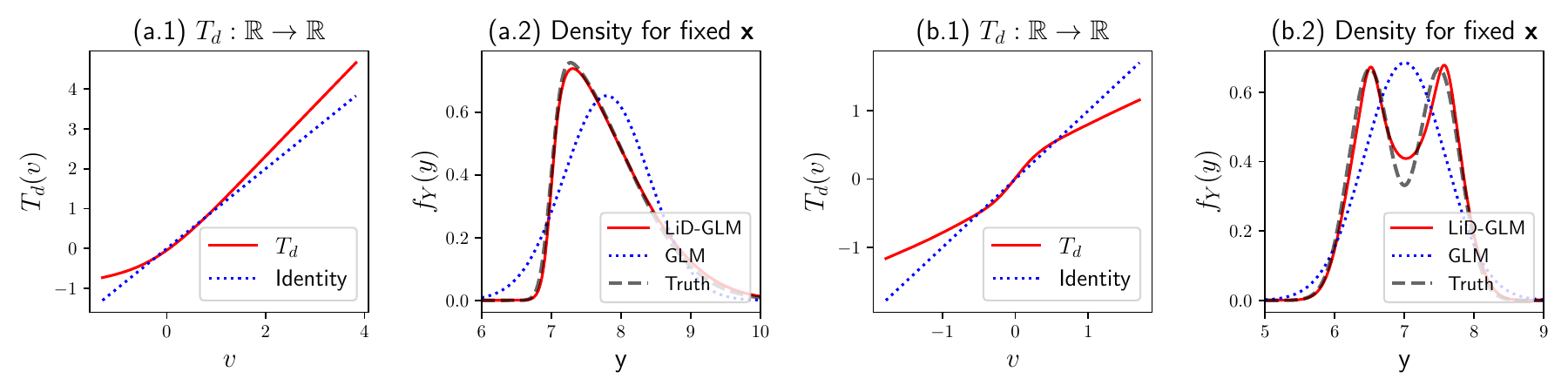}
	\caption{$T_d$ in two synthetic examples as in Figure~\ref{fig: skewEx} where the conditional distribution was now either skewed (Figs.~(a.1, a.2)) or multimodal (Figs. (b.1, b.2)). Shown are the transformations $T_d$ learned by LiD-GLM in both cases (Figs.~(a.1, b.1)) and the resulting conditional distributions (Figs.~(a.2, b.2)). An arbitrary fixed covariate was chosen for all plots as $T_d$ is mean-independent. One can see that, unlike the GLM, LiD-GLM is able to learn skewed and multimodal distributions with an (unimodal, symmetric) normally distributed latent variable.}
	\label{fig: interpret-TdEx}
\end{figure}

\section{Practical Applications}\label{sec: applications}
In this section, we use a nested cross-validation (see Supplement \ref{append: appl-nestedCross}) to compare the performance of LiD-GLM to other models on different datasets and we illustrate the previously introduced interpretation methods. The datasets were chosen to represent a variety of realistic settings, with the ``Cars'' dataset being small and moderately nonlinear, the ``Insurance'' dataset containing strongly skewed residuals and ``Stroke'' being a relatively large and clearly nonlinear dataset. The datasets also cover both the binary and the continuous case. The code for all experiments can be found under \url{https://github.com/tsplittg/LiD-GLM-paper-code}.

\subsection{``Cars'' Application}\label{sec: applications-cars}
The ``Cars'' dataset \citep{Bohanec1988KNOWLEDGEAA, bohanec_car_1988} describes a car's fuel consumption (numeric, mpg) dependent on several independent variables. From these, we selected ``cylinders'' (ordinal), ``horsepower'' (numeric), ``acceleration'' (numeric), ``model\_year'' (numeric) and ``origin'' (nominal, dummy-encoded) and normalized all non-binary covariates. To stabilize the density estimation, we also normalized the target variable.
We also removed entries with missing values (6 total) as well as entries with either three or five cylinders, because these classes are very small (4 or 3 entries respectively).
In total, 13 entries were removed of the 398 originally contained in the dataset.

\begin{table}[t]
	\centering
	\small
	\tabcolsep=0.11cm
	\begin{tabular}{c|r|r|r|r|r|r|r|r|r}
		\toprule
        $k$ & $L_p^b$& $||.||_q$  & lr ($10^x$) & wid. $T_p$& dep. $T_p$ &blk. $T_p$  & act.\ fkt.\ $T_p$ & $\beta$ frozen & test NLL\\
		\midrule
        \midrule
        $\mathcal{H}$ & 
        {\renewcommand{\arraystretch}{0.5}\begin{tabular}{@{}r@{}} 0.01- \\ 1.0 \end{tabular}} &
        1, 2 & -5, -3 & 
        {\renewcommand{\arraystretch}{0.5}\begin{tabular}{@{}r@{}} 9,18, \\ 36 \end{tabular}} &
        {\renewcommand{\arraystretch}{0.5}\begin{tabular}{@{}r@{}} 1, 2, \\ 3 \end{tabular} }&
        1,2 &
        {\renewcommand{\arraystretch}{0.5}\begin{tabular}{@{}r@{}} GroupSort, \\ ReLU \end{tabular}} &
        {\renewcommand{\arraystretch}{0.5}\begin{tabular}{@{}r@{}} yes, \\ no \end{tabular} }& \\   
        \midrule
        \midrule
		
		0  & 0.65 & 2 & $-3$ & 9  & 3 & 2 & GroupSort & yes &  0.310\\
		1  & 0.7  & 2 & $-3$ & 36 & 3 & 1 & GroupSort & no & 0.327 \\
		2  & 0.8  & 1 & $-3$ & 9  & 1 & 1 & GroupSort & no & 0.520 \\
		3  & 0.7  & 2 & $-3$ & 9  & 2 & 2 & GroupSort & yes &  0.406 \\
		4  & 0.9  & 1 & $-3$ & 9  & 3 & 1 & GroupSort & no & 0.551\\

		\bottomrule
	\end{tabular}
	\captionof{table}{Optimal hyperparameters for the ``Cars'' data found for LiD-GLM without $T_d$ on each split $k$. Hyperparamters were chosen from the range (for $L^b_p$) and the values given in row $\mathcal{H}$.  
    Listed are the bound $L_p^b$ of $L(\nu_p)$, the used $p$-norm ($||.||_q$), the learning rate (lr), the width (wid.) and number (dep.) of hidden layers in $T_p$, number of residual blocks in $T_p$ (blk.), whether or not the linear coefficients $\beta$ were frozen during training, activation function used (act.\ fkt.) and test negative loglikelihood (test NLL).}
	\label{tab: cars-topConfigs_no_T_d}
\end{table}

\subsubsection{Performance Comparison}\label{sec: applications-cars-perform}
We show in Figure \ref{fig: appl-cars_perf-loglike} the results of the nested cross-validation for LiD-GLMs with and without distributional correction $T_d$, a traditional GLM and a simple SDDR model \citep{rugamer_semi-structured_2024}, where the identity link and a Normal target distribution were chosen for all models.
As can be seen, the nonlinear LiD-GLM achieves better performance than the GLM model, but also varies more between splits. The inclusion of a distributional correction $T_d$ does not seem to affect mean performance but further increases its variance. The more highly nonlinear SDDR model performs best on average, but only by a narrow margin compared to LiD-GLM; its performance also fluctuates much more strongly between splits than that of the other models. 

In Table~\ref{tab: cars-topConfigs_no_T_d} the optimal hyperparameters found for LiD-GLMs without distributional correction in each split are listed. The corresponding table for LiD-GLMs with $T_d$ can be found in Supplement~\ref{suppl: addRes-cars}.
From Table~\ref{tab: cars-topConfigs_no_T_d} one can see that relatively small networks were generally preferred by the hyperparameter optimization in this application; most splits favor a hidden layer size of only nine nodes with deeper networks with three hidden layers being favored slightly. The preferred activation function for $T_p$ in all splits was GroupSort whereas the results about optimal norm, number of blocks and frozen weights are inconclusive.

\begin{figure}[t]
	\centering
	\begin{subfigure}[t]{0.45\textwidth}
		\includegraphics[height=9cm]{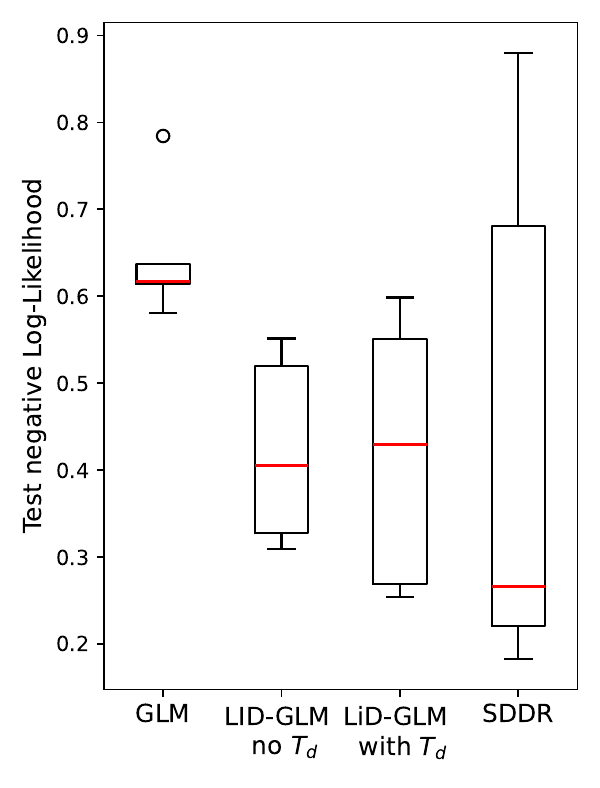}
		\caption{Comparison of test mean negative Log-Likelihoods in the ``Cars'' data example for different models with optimized hyperparameters over $k=5$ splits.}
		\label{fig: appl-cars_perf-loglike}
	\end{subfigure}	
    \hspace{1 cm}
    \begin{subfigure}[t]{0.45\textwidth}
    	\includegraphics[height=9cm]{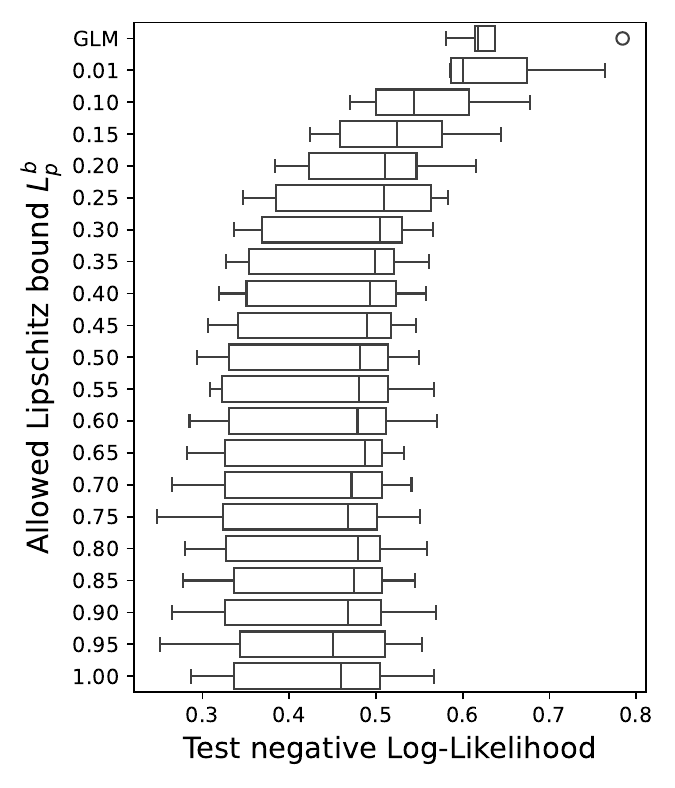}
    	\caption{Boxplots visualizing the impact of the allowed Lipschitz constant on model performance (without $T_d$) pooled over all splits and the top 10 model hyperparameter combinations.}
    	\label{fig: appl-cars_perf-lipDep}
    \end{subfigure}
	\caption{Results for the nested cross validation on the ``Cars'' data..}
\end{figure}

It is interesting that in all splits the allowed Lipschitz bound was not fully utilized and an optimal bound seemed to be between 0.7 and 0.9. We study 
the implications of the bound on $L(\nu_p)$ further by showing in Figure~\ref{fig: appl-cars_perf-lipDep} the test negative mean loglikelihood (NLL) of LiD-GLM (without $T_d$) from the nested cross-validation in dependence on the bound $L_p^b$. Like one would expect, the LiD-GLMs' performance is generally increasing with increasing $L_p^b$, as the model goes from a GLM (listed for reference, $L_p^b = 0$) to an unconstrained NN ($L_p^b\to\infty)$ (see also~\citep[Sec. 5.7]{gouk_regularisation_2021}). One can however observe diminishing returns: the improvement in test NLL for smaller Lipschitz constants is considerably steeper than for larger constants. 

This data therefore appears well suited to be modeled by LiD-GLM. A moderate nonlinearity causes measurable performance benefit over the GLM and diminishing returns already for small $L_p^b$ in Figure~\ref{fig: appl-cars_perf-lipDep} mean that a LiD-GLM which preserves interpretability well is also similar in performance to an unconstrained model.

\subsubsection{Model Interpretation}

We illustrate our interpretation methods for a single LiD-GLM model fit on the entire cars dataset (using an 80/20 train/validation-split to enable early stopping).

Based on the results from Section~\ref{sec: applications-cars-perform}, we use a LiD-GLM model without $T_d$. We therefore do not need to normalize the target variable in this section.

We use one residual block for $T_p$, with two hidden layers, each with width 12, the GroupSort(2) activation function and set $L_p^b= 0.99$. We do not freeze the linear coefficients $\beta$ during training. This results in a validation NLL of 2.38 compared to 2.70 for the GLM.

As a first step of interpretation, we visualize the marginalized dependence of the predicted mean on single covariates via partial dependence plots (PDPs). We show PDPs for most covariates in Figure \ref{fig: applications-cars-allPDPs}. The variable ``horsepower'' is shown in Figure \ref{fig: appl-cars-horsepowerPDP}, where we also stratify by the levels of ``cylinders'' and compare the PDP for the LiD-GLM to those of the linear model and unconstrained NN. 

\begin{figure}[t]
	\begin{subfigure}[b]{0.325\textwidth}
		\caption{Categorical PDPs}
		\centering
		\includegraphics[width=\textwidth]{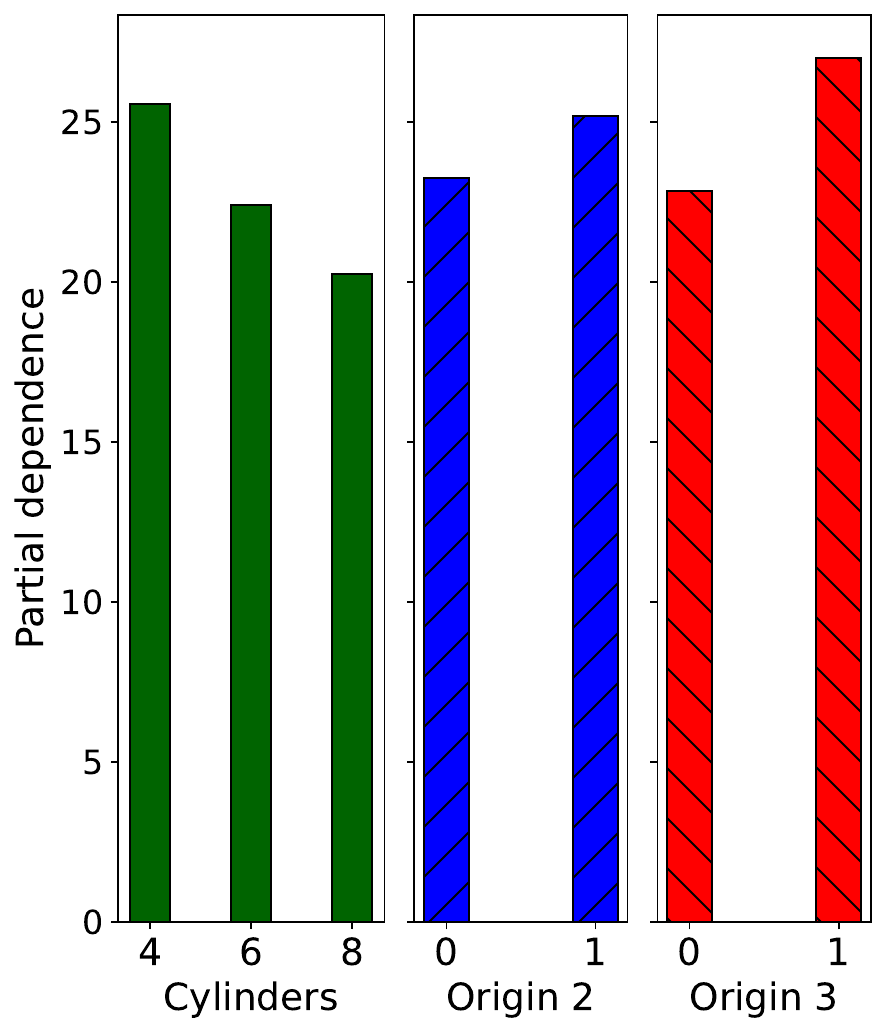}
	\end{subfigure}
	\begin{subfigure}[b]{0.325\textwidth}
		\caption{PDP of acceleration}
		\centering
		\includegraphics[width=\textwidth]{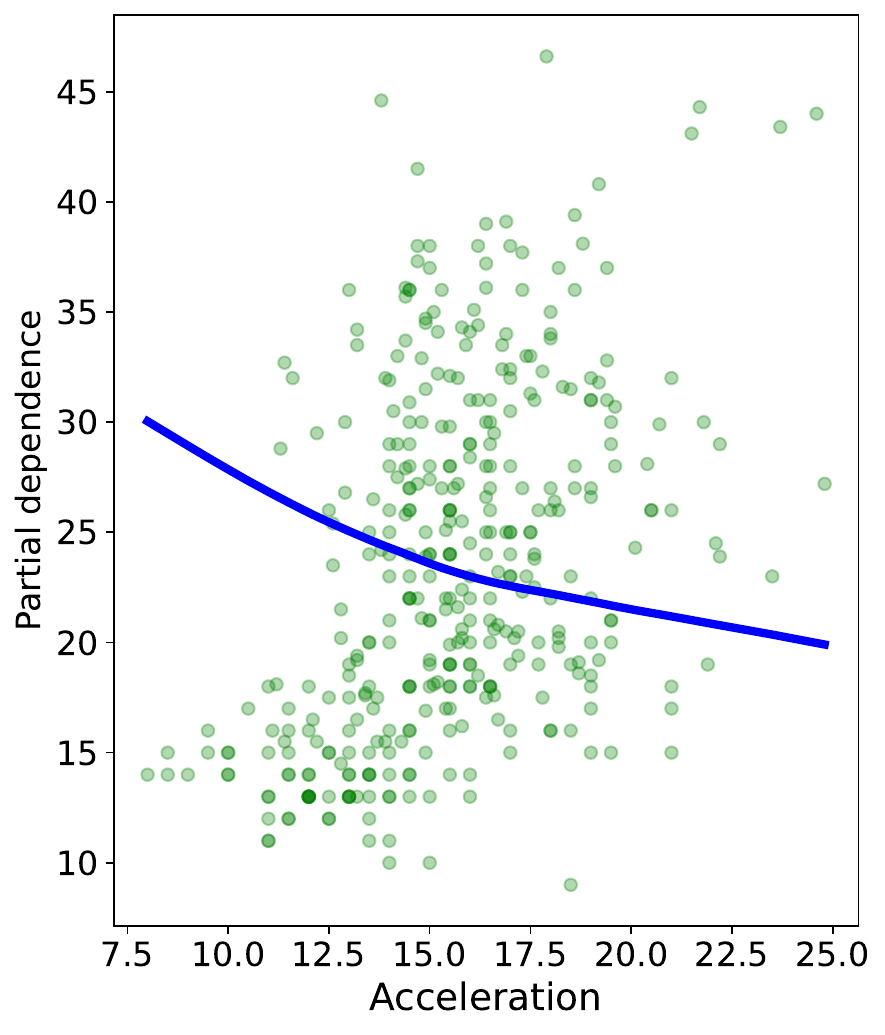}
	\end{subfigure}
	\begin{subfigure}[b]{0.325\textwidth}
		\caption{PDP of model\_year}
		\centering
		\includegraphics[width=\textwidth]{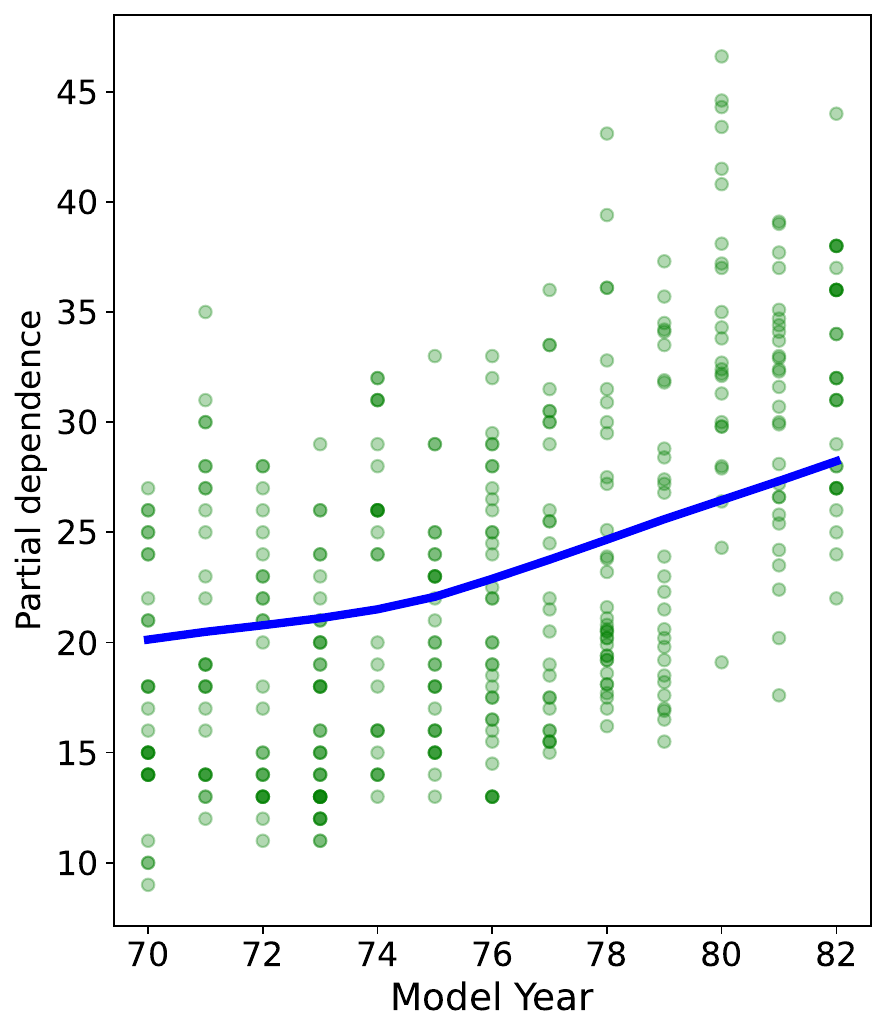}
	\end{subfigure}
	\caption{PDPs for a LiD-GLM trained on the cars data set. For numeric covariates, a scatterplot shows the training data.}
	\label{fig: applications-cars-allPDPs}
\end{figure}
\begin{figure}[t]
	\centering
	\includegraphics[width=\textwidth]{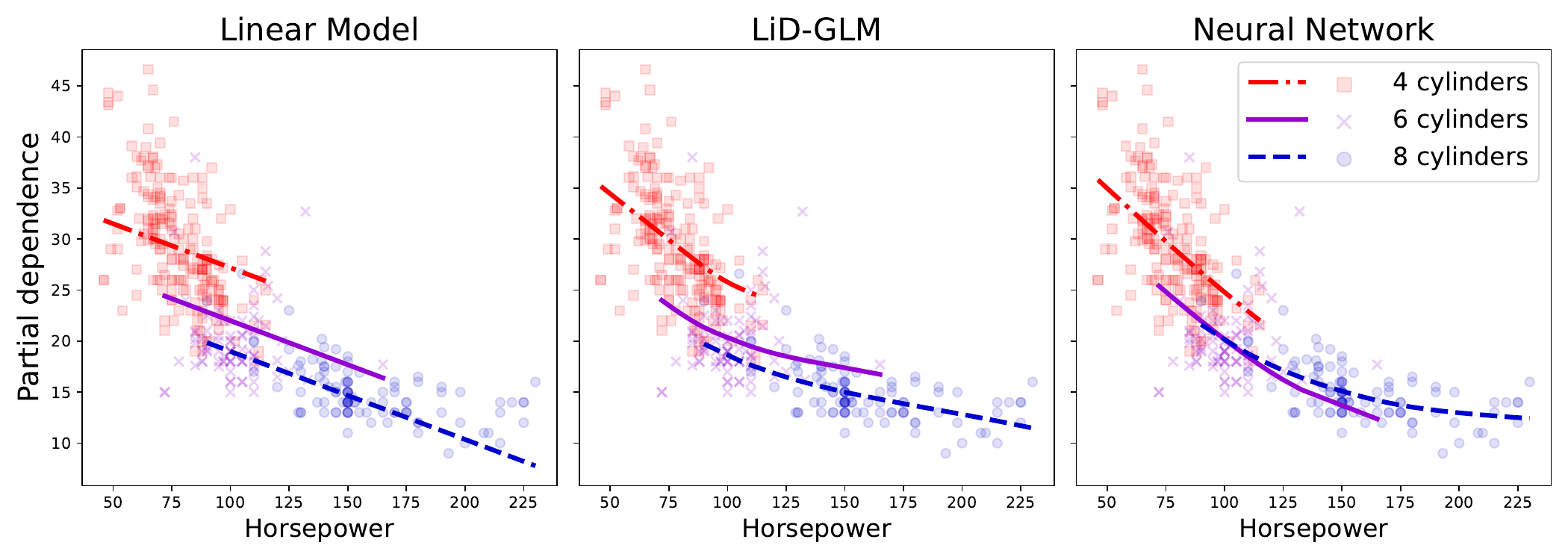}
	\caption{PDPs of ``horsepower'' for different models trained on the cars dataset, each stratified by number of cylinders. A scatterplot shows the training data.}
	\label{fig: appl-cars-horsepowerPDP}
\end{figure}

As can be seen from the PDPs, the LiD-GLM learns a certain degree of nonlinearity seen in the data. In Figure \ref{fig: appl-cars-horsepowerPDP}, we see that the LiD-GLM and NN learn besides a nonlinear trend also the interaction between horsepower and cylinders that cannot be learned with the (additive) GLM. However, the stronger interaction learned with NN appears somewhat suspicious as it implies that cars with eight cylinders are more fuel efficient than cars with six cylinders. 

\begin{table}[t]
\centering
	\begin{tabular}{lrrrrrrr}
		\toprule
        &bias&cylinders&horsepower&acceleration&model year & origin 1 & origin 2\\
        \midrule
        LM & 22.51 &-2.46 & -3.29 & -1.41 & 2.47 & 2.13 & 4.24 \\
        LiD-GLM & 22.39 & -3.0 & -2.49 & -1.31 & 2.57 & 2.77 & 4.82\\
        $R^2$ & - &  0.99 & 0.97 & 0.99 & 0.99 & 0.96 & 0.94 \\
        \bottomrule

    \end{tabular}
    \caption{Coefficient table for LiD-GLM trained on the cars data. Shown are the coefficients of the Linear Model and the LiD-GLM (including PHO) and the coefficients of determination $R^2$.}
	\label{tab: appl-cars-coeffTable}
\end{table}

In Table~\ref{tab: appl-cars-coeffTable}, we list the coefficients $\beta$ learned by the initial linear model
and 
the LiD-GLM (including PHO), which are relatively similar even though $\beta$ was not frozen during the training of LiD-GLM.
We also list the $R_i^2$ as described in Section~\ref{sec: interpret-Tp},
which are all very close to one. This indicates that the transformed variables are close the original covariates and hence the $\beta_i$ can largely be interpreted as in the GLM. The comparatively slightly smaller $R_i^2$ of "horsepower" seems to match the nonlinearities in the corresponding PDP; the smaller $R^2$ of the origin variables could well be caused by learned interaction effects.

As the network contained in $T_p$ is relatively small and only contains one residual block, we can also quantify the influence of $\nu_p$ more precisely. Methods exist \citep{bhowmick_lipbab_2021,splittgerber_exlipbab_2026} to exactly compute the Lipschitz constant of piecewise linear NNs. Applying the ExLipBaB algorithm results in an exact global Lipschitz constant of $L(\nu_p)=0.522$, whereas the bound $L^b_p$ was $0.99$. This means that $T_p$ did not exhaust its allowed Lipschitz constant and the ``majority'' of the model lies in the linear part. This is in line with behavior in Figure \ref{fig: appl-cars_perf-lipDep} - the model did not need to exhaust the Lipschitz constant for an adequate fit and higher allowed constants therefore yield little benefit.

\subsection{``Insurance'' Application}\label{sec: applications-insurance}

To test LiD-GLM on data on which the assumption of a Gaussian residual distribution is likely erroneous, we consider a synthetic dataset from \citep{lantz_machine_2019, insurance_data} which lists the medical expenses an insurance is charged for 1338 clients for a given year based on basic patient information (Supplement \ref{append: appl-detailsInsurance}). We dummy-encode the categorical covariate ``region'', deleting the reference, and normalize the numeric covariates and target variable.

\subsubsection{Performance Comparison}\label{sec: applications-insurance-perform}

Again, we compare a LiD-GLM with and without distributional correction, a traditional GLM and an SDDR model and again use the identity link and a Normal target distribution for all models.
We show the test performance of each model in figure \ref{fig: appl-insurance_perf-loglike}; the optimized hyperparameter configurations for LiD-GLM are listed in the supplement.

Like in Section~\ref{sec: applications-cars}, the nonlinearization of the mean estimation results in a measurably better model performance of LiD-GLM (even without $T_d$) compared to the base GLM. The SDDR model also has better performance on average but suffers from instability in a single data split. Also, as one can conclude from figure \ref{fig: appl-insurance_perf-lipDep}, a moderately nonlinear model is likely again sufficient in this data example. 
However, the inclusion of a distributional correction $T_d$ in this application improves LiD-GLMs' performance by a considerable margin compared to models without distributional correction. In the next subsection, we investigate the distributional correction that is learned on this data and whether it is likely to be helpful in practice.

\begin{figure}[t]
	\centering
	\begin{subfigure}[t]{0.45\textwidth}
		\includegraphics[height=9cm]{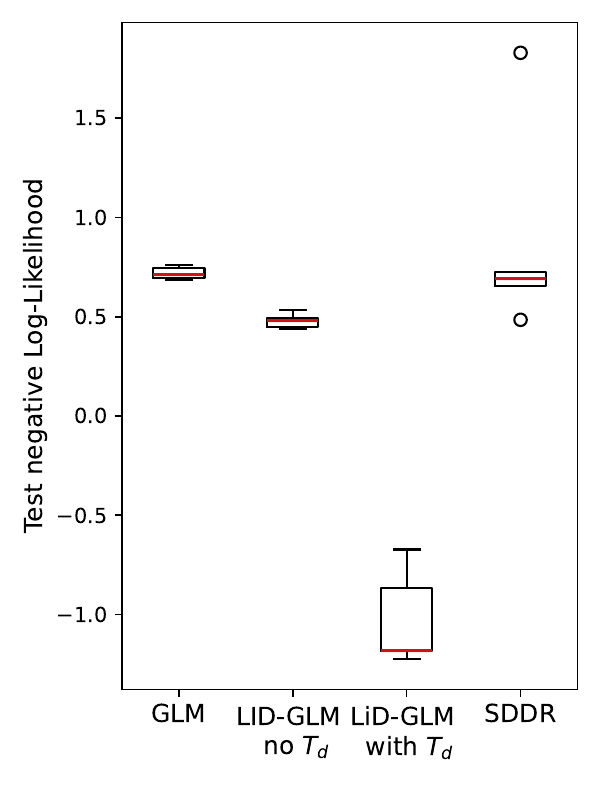}
		\caption{Comparison of test mean negative Log-Likelihoods for different models with optimized hyperparameters over $k=5$ splits.}
		\label{fig: appl-insurance_perf-loglike}
	\end{subfigure}	
    \hspace{1 cm}
    \begin{subfigure}[t]{0.45\textwidth}
    	\includegraphics[height=9cm]{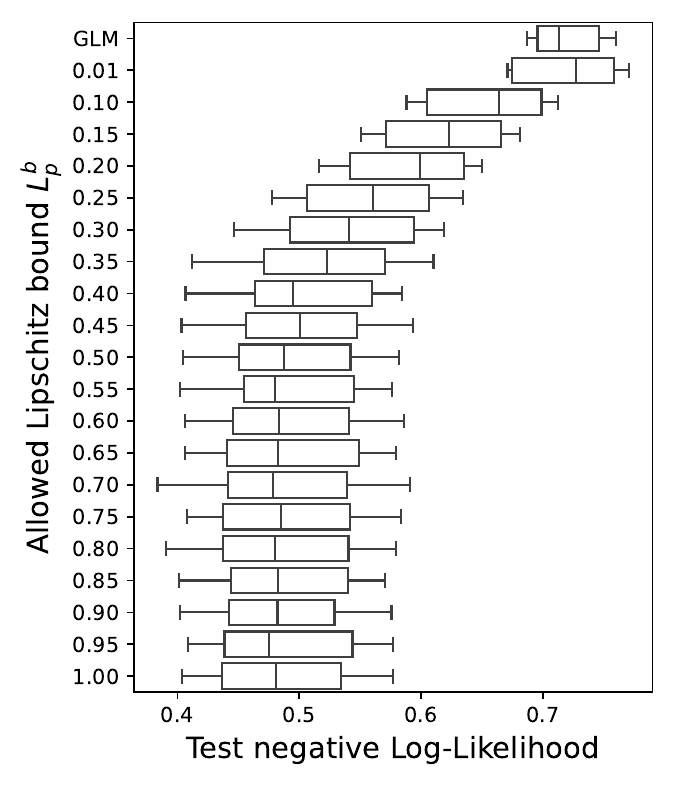}
    	\caption{Impact of the allowed $L_p^b$ on model performance (without $T_d$) (all splits and top 10 model hyperparameter combinations pooled).}
    	\label{fig: appl-insurance_perf-lipDep}
    \end{subfigure}
	\caption{Results for the nested cross validation on the ``Insurance'' data..}
\end{figure}

\subsubsection{Model Interpretation}\label{sec: applications-insurance-interpr}

In the interpretation of LiD-GLM on this data, we focus on the distributional correction that is learned by the model. For this, we train a LiD-GLM that contains both $T_p$ (1 block, 3 hidden layers, width=9, ReLU-activation, $L_p^b=0.99$) and $T_d$ (3 blocks, each 1 hidden layer, width=6, Swish-activation, $L_d^b=2.99$) and is trained on a single train/validation/test split of the insurance data. The model is able to achieve a mean test negative loglikelihood of -0.7, which is in line with the performance achieved in Section~\ref{sec: applications-insurance-perform}.

\begin{figure}[t]
	\begin{subfigure}[t]{0.48\textwidth}
		\centering
		\includegraphics[width=\textwidth]{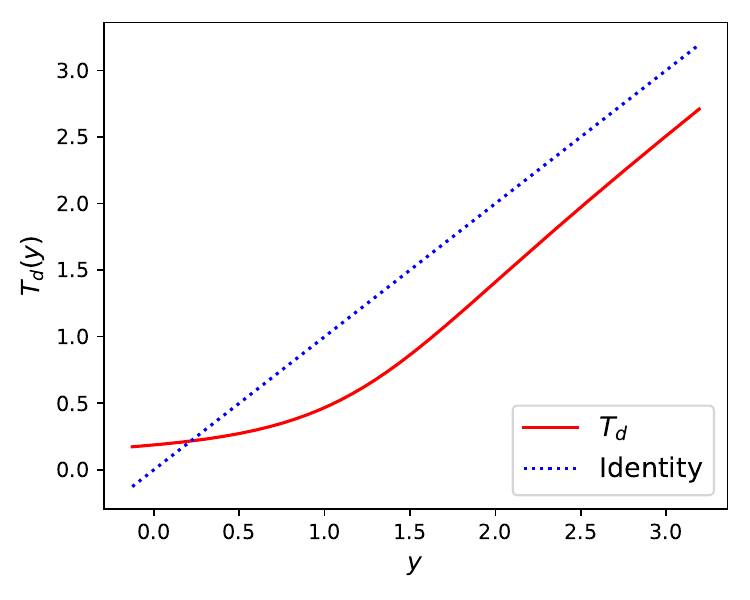}
		\subcaption{$T_d$ as a function $\mathbb{R}\to\mathbb{R}$}
		\label{fig: applications-insurance-distrCorrect-asFctn}
	\end{subfigure}
	\begin{subfigure}[t]{0.48\textwidth}
		\centering
		\includegraphics[width=\textwidth]{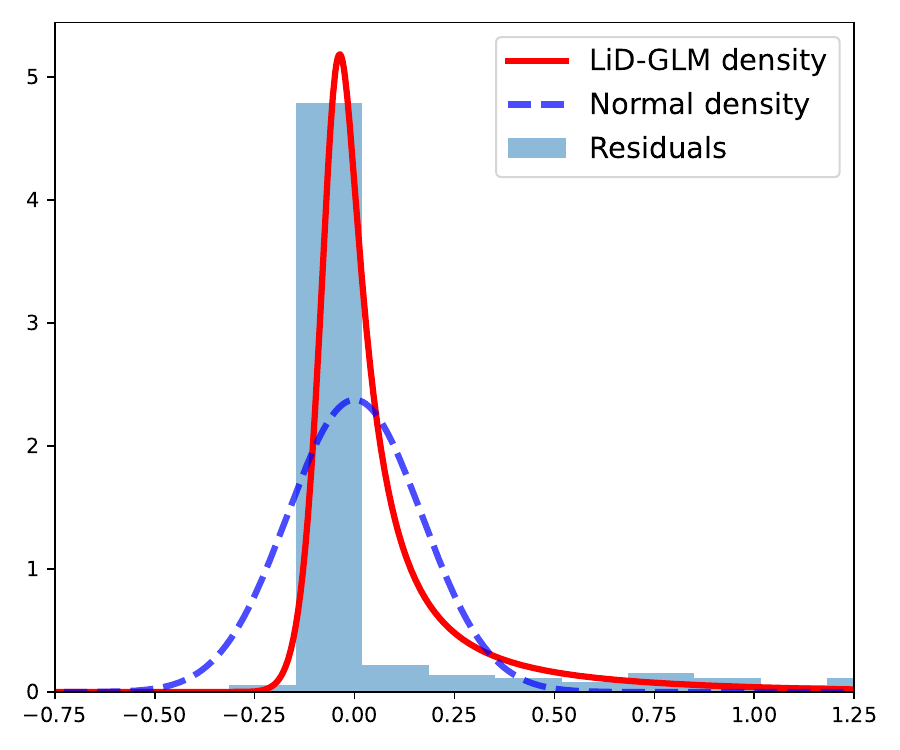}
		\subcaption{Conditional distribution learned by LiD-GLM compared to Residuals and an estimated Normal distribution.}
		\label{fig: applications-insurance-distrCorrect-density}
	\end{subfigure}
	\caption{Distribution correction learned by LiD-GLM on the insurance dataset.}
	\label{fig: applications-insurance-distrCorrect}
\end{figure}

As can be seen in Figure~\ref{fig: applications-insurance-distrCorrect}, LiD-GLM learns a skewed conditional distribution which seems to fit the residual distribution quite well, especially compared to a Normal distribution which cannot capture the asymmetric distribution obviously contained in the residuals. As can be seen in Figure~\ref{fig: applications-insurance-distrCorrect-asFctn}, the function $T_d$ is similar to the one shown in Figure~\ref{fig: interpret-TdEx} (a.1), which was our prototypical example for a correction on skewed data.

Further model interpretation details on aspects not concerning distributional correction are listed in Supplement~\ref{suppl: addRes-insurance}.

\subsection{``Stroke'' Application}\label{sec: applications-stroke}

Lastly, for an example of a large and complex binary dataset, we use the ``Stroke'' dataset \citep{stroke_data}. In it, the occurrence of a stroke in a patient is provided as a binary target variable together with basic patient information. We dropped three entries with NA values and 58 entries with negative ``age''-values, after which the dataset had 40849 entries. We also dropped the single categorical covariate ``work\_type'', which has some dubious entries. After this, we were working on nine covariates, and again normalized the non-binary covariates.

\subsubsection{Performance Comparison}\label{sec: applications-stroke-perform}
Again, we compare a LiD-GLM, a traditional GLM and an SDDR model \citep{rugamer_semi-structured_2024}, where we this time choose the logit link and a Bernoulli target distribution for all models. On this data, we also compare to a basic NN with a Sigmoid last layer, which, in the binary special case, outputs a probability (-distribution) and is equivalent to an unconstrained Deep GLM. Note that the distributional correction is not applicable in the binary setting. 

We show the search space of a nested cross-validation and optimal hyperparameters for LiD-GLM in Table~\ref{tab: stroke-topConfigs_large_Lip} and the performance comparison between the models in Figure~\ref{fig: appl-stroke_perf-loglike}, where we also separately show the best LiD-GLM for each split if $\beta$ is frozen during training, or if a small $L_p^b=3.29$ is chosen. The latter two options result in a worsened goodness-of-fit, but can be advantageous for model interpretation.

In Figure~\ref{fig: appl-stroke_perf-loglike-LipDep}, we show the test NLL in dependence of the bound $L_p^b$. In this data example, 
one sees diminishing returns only for relatively large bounds. A bound $L_p^b\approx 14$ seems to be optimal, which is further supported by Table~\ref{tab: stroke-topConfigs_large_Lip}, which also shows a preference towards larger networks with more residual blocks than in Section~\ref{sec: applications-cars-perform}.
A LiD-GLM with a relatively small bound, e.g. $L_p^b = 1.16$, however also achieves significant performance benefit over regular logistic regression. Still, this emphasises the compromise inherent to LiD-GLM: a researcher can choose a higher $L^b_p$ to improve performance or intentionally use a more interpretable but less expressive model.

\begin{figure}[t]
	\centering
	\begin{subfigure}[t]{0.45\textwidth}
        \centering
    	\includegraphics[width=\textwidth]{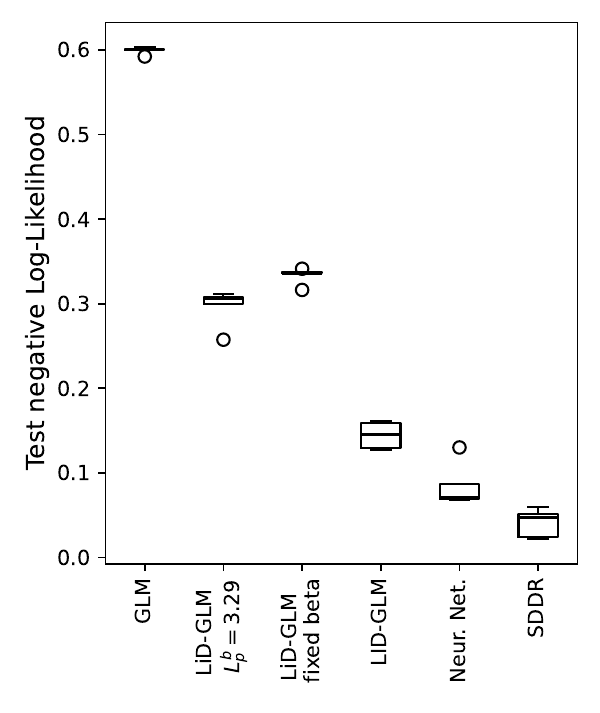}
    	\subcaption{Comparison of test negative Log-Likelihoods in the ``Stroke'' data example for different models with optimized hyperparameters over $k=5$ splits.}
    	\label{fig: appl-stroke_perf-loglike}
	\end{subfigure}
	\hfill
	\begin{subfigure}[t]{0.45\textwidth}
		\centering
		\includegraphics[width=\textwidth]{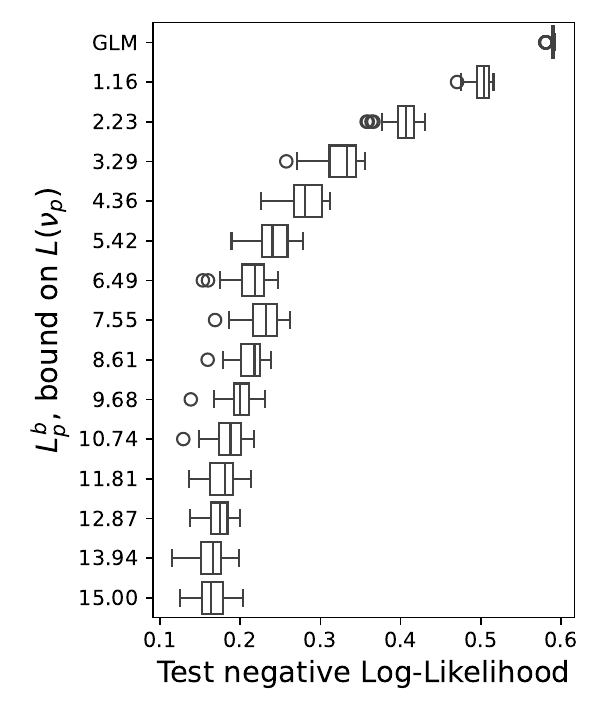}
		\subcaption{Boxplots visualizing the impact of the allowed Lipschitz constant on model performance pooled over all splits and the top 10 model hyperparameter combinations.}
		\label{fig: appl-stroke_perf-loglike-LipDep}
	\end{subfigure}
	\caption{Results for the nested cross validation on the ``Stroke'' data.}
	\label{fig: appl-stroke_perf}
\end{figure}

\begin{table}
    \begin{center}
	\begin{tabular}{c|r|r|r|r|r|r|r|r|r|r|r}
		\toprule
		$k$ & $L_p^b$ & $||.||_q$ & lr ($10^x$) & wid. $T_p$ & dep. $T_p$ &  blk. $T_p$ & act. fkt. $T_p$  & $\beta$ frozen & test NLL\\
		\midrule
        \midrule
        $\mathcal{H}$ & 
        {\renewcommand{\arraystretch}{0.5}\begin{tabular}{@{}r@{}} 0.1- \\ 15.0 \end{tabular}} &
        1,2 & 
        {\renewcommand{\arraystretch}{0.5}\begin{tabular}{@{}r@{}} -5, -4, \\ -3 \end{tabular}} &
        {\renewcommand{\arraystretch}{0.5}\begin{tabular}{@{}r@{}} 18, 36, \\ 48 \end{tabular}} &
        {\renewcommand{\arraystretch}{0.5}\begin{tabular}{@{}r@{}} 3, 4, \\ 5, 6 \end{tabular}} &
        
        {\renewcommand{\arraystretch}{0.5}\begin{tabular}{@{}r@{}} 3, 4, \\ 5 \end{tabular}} &
        {\renewcommand{\arraystretch}{0.5}\begin{tabular}{@{}r@{}} ReLU \\ GroupSort \end{tabular}} &
        yes, no &\\
        \midrule
        \midrule
		0  & 15.00 & 2 & -3 & 36 & 6 & 5 & GroupSort & no & 0.162 \\
		1  & 15.00 & 2 & -3 & 36 & 6 & 4 & GroupSort & no & 0.127 \\
		2  & 13.94 & 2 & -3 & 48 & 5 & 5 & GroupSort & no & 0.159 \\
		3  & 12.87 & 2 & -3 & 48 & 6 & 4 & GroupSort & no & 0.146 \\
		4  & 13.94 & 2 & -3 & 36 & 5 & 5 & GroupSort & no & 0.129\\        
		\bottomrule
	\end{tabular}
	\caption{Optimal hyperparameters for LiD-GLM found for each split $k$ in the ''Stroke'' data example. For a further table description, see Table~\ref{tab: cars-topConfigs_no_T_d}.}
	\label{tab: stroke-topConfigs_large_Lip}
    \end{center}
\end{table}

\subsubsection{Model Interpretation}\label{sec: applications-stroke-interpr}

In Table~\ref{tab: appl-stroke-coeffTable}, we again show the coefficients  of a LiD-GLM model with 6 hidden layers of width 48, distributed into 5 residual blocks with $L^b_p=13.99$. To help with interpretation, we froze the linear coefficients during training and the model therefore only achieved a validation NLL of 0.369 on a single split (80/20 train/validation) of the stroke data (results for an unconstrained model are given in Supplement \ref{suppl: addRes-stroke}). As can be seen from the listed $R^2_i$-values, the model incurs more nonlinearity than in ``Cars'' application, which is to be expected from Figure \ref{fig: appl-stroke_perf-loglike-LipDep}. The covariates ``age'' and ``avg. glucose level'' seem to remain relatively unchanged, which matches the PDP plots shown in Figure \ref{fig: appl-stroke-pdp}. 

In contrast, the PDP plot for ``BMI'' is relatively nonlinear, despite what the $R^2_i$ which is exactly 1 (rounded to two digits) would suggest. This can be partially explained by the fact that a large portion of the PDP plot lies outside the data dense region between the 10 and 90 percent quantiles of ``BMI''. But to explain this phenomenon further, we show in Figure \ref{fig: appl-stroke-pdpTransfVars} the PDP plots of the transformed formely binary covariates w.r.t.\ the ``BMI'' variable. As can be seen, ``BMI'' seems to have a nonlinear influence on all binary covariates. These learned interactions likely also explain the relatively small $R^2_i$-values of other covariates. Especially in the cases of ``hypertension'', ``heart disease'' and ``ever married'', in which the PDP plot does not even intersect the average value of the transformed covariate, strong interactions are likely occurring. In this way the LiD-GLM can also be leveraged to detect possible interaction effects for downstream variable selection.

\begin{table}[t]
\centering
	\begin{tabular}{lrrrrrrrrrr}
		\toprule
        & $\beta_0$ & $\beta_1$ &  $\beta_2$ & $\beta_3$ & $\beta_4$ & $\beta_5$ & $\beta_6$ & $\beta_7$ & $\beta_8$ & $\beta_9$ \\
        \midrule
        LM & -1.0 & -0.41 & 0.09 & 1.23 & 1.13 & 0.87 & 0.13 & 0.42 & -0.16 & 0.13 \\
        LiD-GLM & -2.06 & -0.54 & 0.16 & 1.8 & 1.64 & 1.6 & 0.34 & 0.61 & -0.46 & 0.09\\
        $R^2_i$ & - & 0.96 & 0.99 & 0.64 & 0.54 & 0.79 & 0.98 & 0.98 & 1.0 & 0.85 \\

        \bottomrule

    \end{tabular}
    \caption{Coefficient table for LiD-GLM trained on the Stroke data. Shown are the coefficients of the Linear Model and the LiD-GLM (including PHO) and the coefficients of determination $R^2$ for the variables $\beta_0$=Intercept, $\beta_1$=Sex, $\beta_2$=Age, $\beta_3$=Hypertension, $\beta_4$=Heart disease, $\beta_5$=ever married, $\beta_6$=Residence type, $\beta_7$= Average glucose level, $\beta_8$=BMI, $\beta_9$=smoking status.}
	\label{tab: appl-stroke-coeffTable}
\end{table}

\begin{figure}
	\centering
	\includegraphics[width=\textwidth]{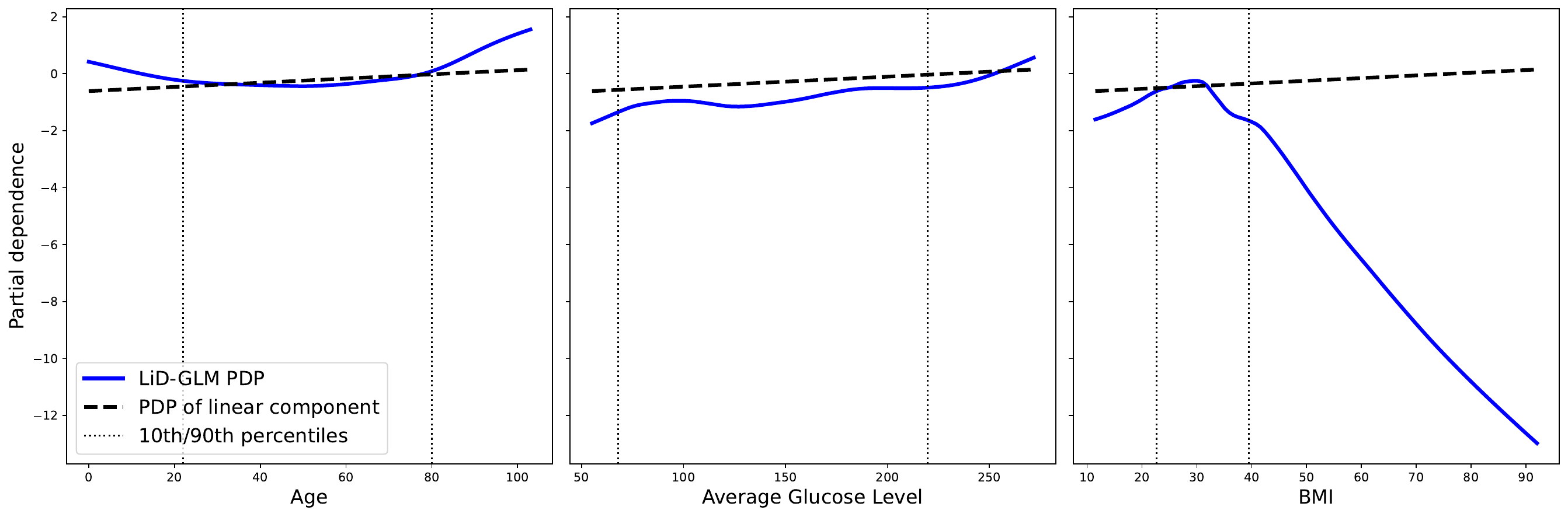}
	\caption{PDP plots of LiD-GLM on the continuous variables in stroke data example w.r.t. predicted logits.}
	\label{fig: appl-stroke-pdp}
\end{figure}

\begin{figure}[t]
	\centering
	\includegraphics[width=0.9\textwidth]{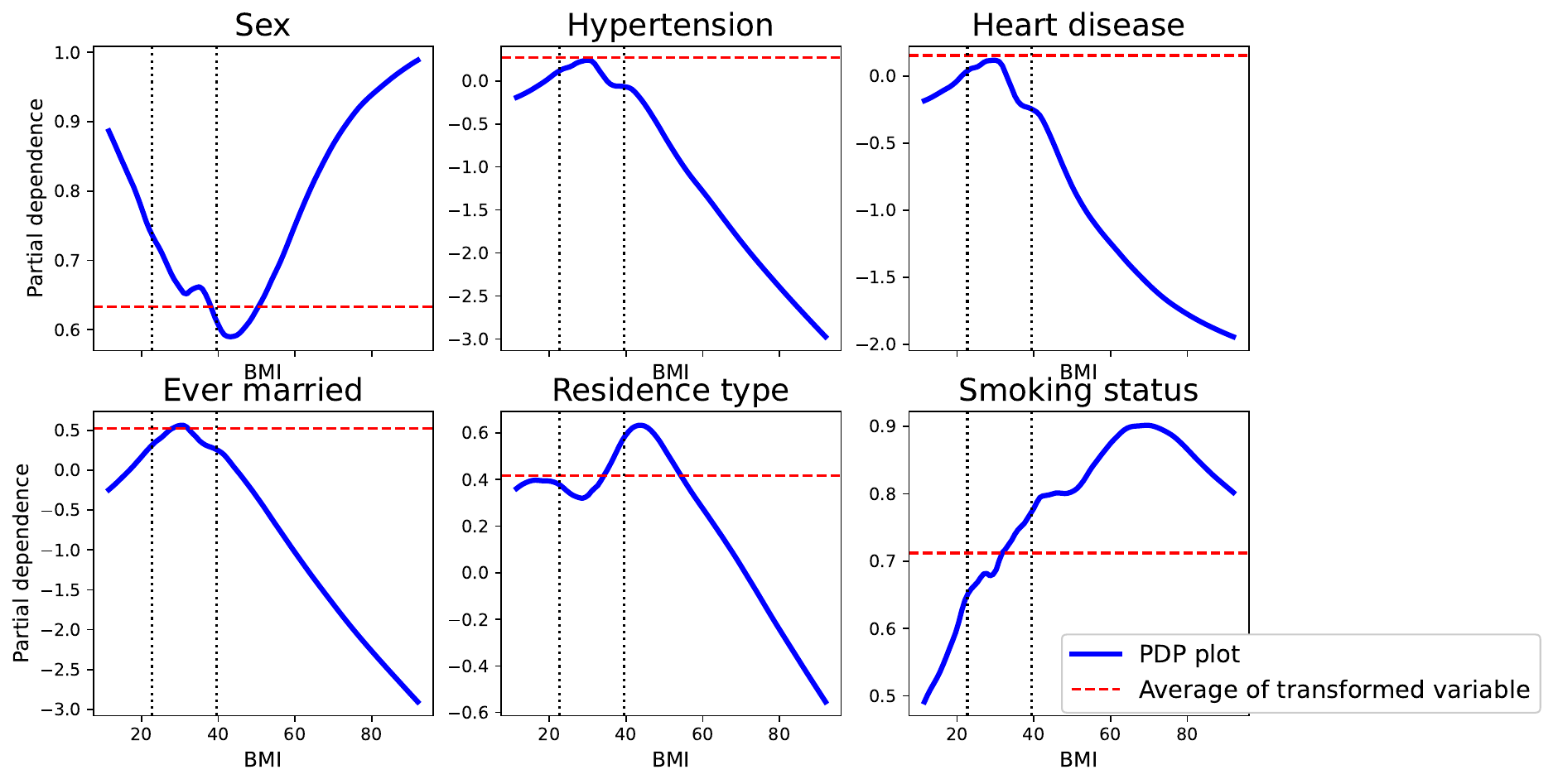}
	\caption{PDP plots of transformation $T_p$ w.r.t.\ ``BMI'' of the former binary covariates.}
    \label{fig: appl-stroke-pdpTransfVars}
\end{figure}

\subsection{PHO Orthogonalization}\label{sec: appl-orth}
We discussed in Section~\ref{sec: ldglm-orthog} how post-hoc orthogonalization (PHO) can counteract a lack of identifiability theoretically. In previous research \citep{rugamer_new_2023}, it was shown that orthogonalization can also drastically reduce the mean squared error (MSE) when reconstructing linear effects from simple synthetic data. We have adopted their approach to also test the behavior of our model first on synthetic and then on real data. Generally, one would expect that improved identifiability should also stabilize parameter estimation. 

\subsubsection{Synthetic Linear Data}
For different numbers $p\in\{1,3,10\}$ of independent variables, we sample a number ${n\in\{50, 100, 1000\}}$ of values $\textbf{X}$ from a $p$-dimensional multivariate normal distribution. We then choose $p$ equidistant coefficients $\beta$ in the interval $[-2.5, 2.5]$ and compute a target array through a linear combination y= $\textbf{x}\cdot \beta$. We also add standard Normal noise to the target. On this synthetic data, we train a Linear Regression model as well as an SDDR and LiD-GLM model and compare in Figure \ref{fig: appl-orth-linEff} the MSEs between the estimated and true linear coefficients from 20 different random seeds. We test the SDDR model and LiD-GLM both with and without PHO; for the SDDR model we also test ONO-orthogonalization \citep{rugamer_semi-structured_2024}. For the SDDR and LiD-GLM, we use networks with the ReLU activation; for the SDDR model with two hidden layers - one with 100 and one with 50 nodes and a dropout of 0.1 after each layer. The LiD-GLM model has two hidden layers with 75 nodes each. As described in previous sections, our recommended workflow would almost always be to initialize LiD-GLM with the weights from a pre-trained GLM. However, to enable further comparison, we also show results for a LiD-GLM that was initialized with random weights. 

The Linear model is, unsurprisingly, able to reconstruct the true linear coefficients with negligible error in all tested scenarios. In our recommended workflow, the LiD-GLM is therefore essentially initialized on the correct coefficients and performs with negligible error as well. We argue that this is not in itself an unfair comparison - indeed, scenarios where the GLM already reflects the general trend of the data well are precisely what our model is constructed for. 
However, given that an SDDR model could also be initialized at a pre-trained GLM, we also compare the two models both initialized on random weights. Without orthogonalization, both models then incur significant error, especially for small $p$ and $n$, likely because of a lack of identifiability. However, a clear gap in performance is recognizable between the two models on this specific task. The constraint enforced on $L(\nu_p)$ in the LiD-GLM presumably already counteracts the lack of identifiability to some degree by encouraging the optimizer to rely on the linear component of the model rather than the nonlinear one. For both models, orthogonalization then generally improves the resulting MSE, which is the same conclusion that was drawn in \citep{rugamer_new_2023}. 

\begin{figure}[t]
	\centering
	\includegraphics[width=\textwidth]{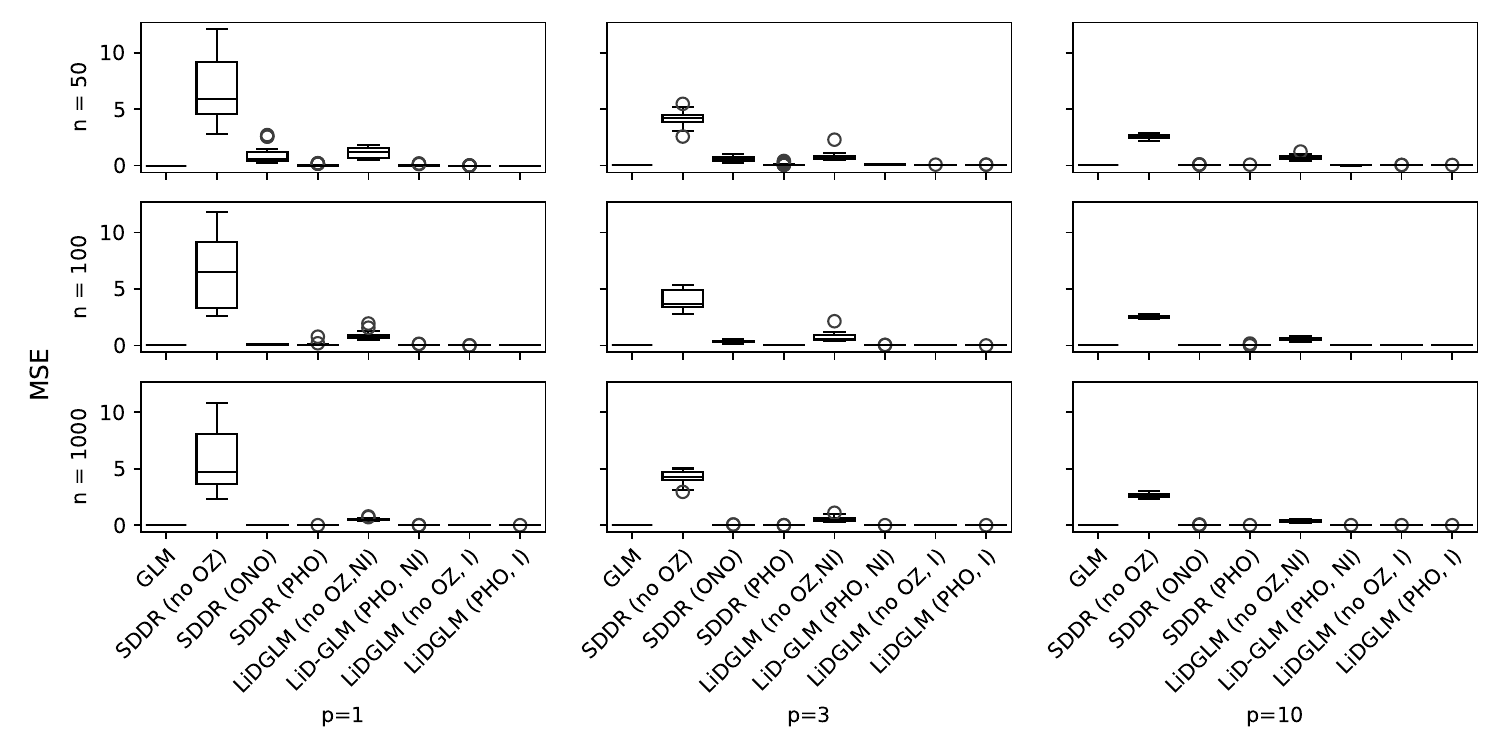}
	\caption{MSE of coefficients learned on synthetic data compared to known truth for: linear model (GLM), SDDR model without orthogonalization (no OZ), with ONO and PHO orthogonalization and LiD-GLM with and without PHO, initialized at linear model (I) and initialized at random weights (NI).}
	\label{fig: appl-orth-linEff}
\end{figure}

\subsubsection{Real Data}

To our knowledge, a discussion on the effect of orthogonalization on the estimated coefficients for SSDR-like models, as in the last subsection, has not yet been had on real data. Obviously, one generally cannot compare an absolute error to some known true effects in a real setting. However, one would expect that orthogonalization, by ensuring identifiability, at least reduces the variance between linear coefficients estimated in different local minima of the optimizer. This would indeed already improve interpretability as a large fluctuation in the linear coefficients for different optimizer paths would put into question any interpretation based on these coefficients.

To test the effect of PHO on the coefficients estimated by a LiD-GLM model in a realistic application, we compare the coefficients with and without PHO learned by the top 20 models (based on validation log-likelihood) in the nested cross-validation of Section~\ref{sec: applications-cars}, since the ``Cars'' data is an ideal application for a LiD-GLM. We show here results for LiD-GLMs without a distributional correction $T_d$, but results for LiD-GLMs with $T_d$ and results for the ``Insurance'' and ``Stroke'' applications from Sections~\ref{sec: applications-insurance} and \ref{sec: applications-stroke} can be found in Supplement \ref{suppl: addRes-Pho}. 
In our comparison, we consider each data split separately. This separates fluctuation in the linear effects' estimation due to randomness in the data from (potentially avoidable) fluctuations that arise from a lack of identifiablity. By considering the top 20 models for each split, we get a realistic spectrum of hyperparameters that could have been chosen in practice.

In Figure \ref{fig: appl-cars_perf-PHO_vis}, we show boxplots of the resulting estimated parameters. For conciseness, only the first three splits are shown, which are representative of the general behavior; results for the other splits can also be found in Supplement \ref{suppl: addRes-Pho}.
In the cars application, the variance in the estimation of all $\beta_i$ is drastically reduced through an application of PHO for all coefficients. Also, the application of PHO generally reduces the absolute coefficient values.

\begin{figure}[t]
    \includegraphics[width=\textwidth]{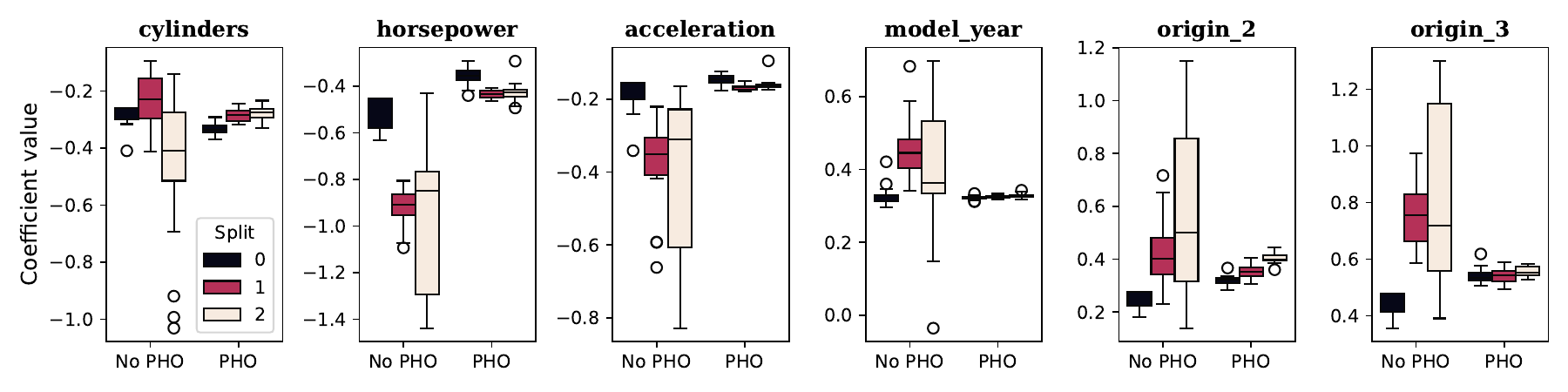}
	\caption{Boxplots showing the LiD-GLM estimates of the linear coefficients  with and without PHO on the first three splits in the nested cross validation on the ``Cars'' dataset.}
	\label{fig: appl-cars_perf-PHO_vis}
\end{figure}

\section{Discussion}\label{sec: discussion}

We introduced an extension of GLMs which addresses two key issues by allowing a flexible nonlinear parameter estimation and also a correction of the a priori assumed distribution family. LiD-GLM achieves both generalizations by utilizing i-ResNets; consequently, the deviation from the GLM, and thereby the degree of nonlinearity of the model, can be rigorously bounded through a user-chosen Lipschitz constraint. This preserves the interpretation of the linear regression coefficients up to a degree that is influenced by the actual Lipschitz constants of the nonlinear transformation applied to the individual terms of the model; it also partially preserves monotonicity. These Lipschitz constants are approximated during training and can be computed exactly post-hoc e.g.\ utilizing methods suggested in \citep{splittgerber_exlipbab_2026}. To also give a more intuitive method to quantify the degree of nonlinearity of individual covariates we utilize an adapted coefficient of determination. This provides further insight into the interpretability of the learned linear coefficients. Furthermore, we discuss identifiability as an important part of interpretability and use a post-hoc orthogonalization to guarantee identifiability. Additionally, we show how the distributional correction learned by LiD-GLM can be easily interpreted.

We showed that on data sets of moderate nonlinearity and non-additivity our models can close the gap between the restrictive GLMs and over-expressive and unconstrained ML methods. We also tested our model internally on a highly linear real clinical dataset. Unfortunately we cannot publish these results for data protection reasons, but we would like to mention that in this case the LiD-GLM almost exactly replicated the GLM. This is not particularly surprising, as LiD-GLM was again initialized with a pre-trained GLM and simply stayed at the same weight, like in Figure \ref{fig: appl-orth-linEff}. It however gives researchers a straightforward tool to check whether a moderate nonlinearization of a plausible GLM would improve model fit. If the answer is no, like in this particular case, then a researcher has additional justification to stick to a GLM. In comparison, a model which is randomly initialized, like a NN usually is, could also be performing badly due to bad initialization. We also observed a better stability of LiD-GLM than for NNs, like was the case in Figure~\ref{fig: appl-orth-linEff}.

It is important to note that the Lipschitz constant, by definition, depends on the scale of the input. We therefore expect normalization, which is a standard practice in statistics and machine learning, to be of special importance in the context of LiD-GLM.
We also observed in our testing that a normalization of the target variable can be a valuable tool to stabilize the learned distributions in cases where a $T_d$ is applied. It is hardly surprising that in cases like the insurance example, where the target data is of order $10^6$ and the likelihood of order $10^{0}$ numerical inaccuracies can arise otherwise.

The class of functions that Lipschitz constrained models like LiD-GLM are capable of approximating depends, among other factors, on 
the choice of the activation function.
\cite{huster_limitations_2019} show that ReLU-NNs with a layerwise constrained
Lipschitz constant of one are not capable of approximating all functions with Lipschitz constant one.
However, it has been shown that NNs using the GroupSort activation function do not share this limitation \citep{anil_sorting_2019}. Future research could clarify which class of functions can be universally approximated with LiD-GLM when using activation functions like GroupSort. 

Given the preserved interpretability of the regression coefficients with LiD-GLM, confidence intervals for these coefficients would also be a valuable tool for researchers. Such methods may utilize the Lipschitz constraint and will likely be based on a bootstrapping technique.

However, at the current state of research, LiD-GLM can already be a valuable tool for researchers often working with GLM who suspect their data might benefit from slightly nonlinear models. A researcher could, for example, test a variety of different Lipschitz constants and, depending on the necessary criteria of interpretability, make an informed decision which model to use or which degree of nonlinearity to allow.

\section{Disclosure Statement}
The authors report there are no competing interests to declare. 

\section{Disclosure Concerning Generative AI}
In the writing of this paper, generative AI was used only as a tool for literature search (ChatGPT ver. 5.x), to find potential errors in the manuscript (Claude 5) and as an auto-complete tool during programming (so-called ``inline suggestions'' via GitHub Copilot (using the ``automatic'' model selection).

\section{Data availability statement}
The datasets that were used in the above experiments are publicly available in the UCI machine learning repository under \doi{10.24432/C5JP48} ( for the``Cars'' dataset), and on Kaggle under \url{https://www.kaggle.com/datasets/mirichoi0218/insurance} for the ``Insurance'' and under \url{https://www.kaggle.com/datasets/prosperchuks/health-dataset/data} for the ``Stroke'' dataset.

%% file: tex_components/supplement.tex
\noindent{\huge\bf Supplementary Material}
\section{Decomposition of i-ResNet}\label{sec: append-iResNet-Decomp}
\begin{lma}\label{lma: append-iResNet-Decomp}
	Any i-ResNet $T$ \citep{behrmann_invertible_2019} can be written in the form 
    \begin{equation}\label{eq: nu1-deconstructedT1}
        T = \text{Id}+\nu
    \end{equation}
    with $L(\nu)<\infty$.
\end{lma}
\begin{proof}
	Any i-ResNet $T$ is, by definition, of the form $T=(T^m\circ\dots\circ T^1)$, where each $T^j, 1\leq j\leq m$ is a residual block $T^j=\text{Id}+g_j$ with $L(g_j)<1$ and thus the Lemma can be proven by induction over $m$.
    For $m=1$, the statement is trivially true, as $T=T^1$ already has form \eqref{eq: nu1-deconstructedT1}.
    
	If the statement has been proven for all i-ResNets with $m-1$ blocks, then a network $T=T^m\circ\dots \circ T^1$ with $m$ blocks can be rewritten as $T= (\text{Id}+g_m)\circ(\text{Id}+\nu)$
	with $L(\nu)<\infty$ and therefore:
	\begin{equation*} 
		T^m\circ\dots \circ T^1 = (\text{Id}+g_m)\circ(\text{Id}+\nu) = \text{Id} + \underbrace{\nu+g_m(\text{Id}+\nu)}_{=:\tilde{\nu}},
	\end{equation*}
	where $\tilde{\nu}$ is a NN with
	$L(\tilde{\nu})\leq L(\nu)+L(g_j)L(\text{Id}+\nu)<\infty$.
	This proves the statement for any i-ResNet with $m$ blocks and Lemma \ref{lma: append-iResNet-Decomp} follows by induction. 
\end{proof}	
	
	\section{Relation between LiD-GLM and SDDR}\label{sec: append-asSDDR}
	In this section, we discuss the close relation between the extension of Section~\ref{sec: nu1} to semi-structured deep distributional regression \citep{rugamer_deepregression_2023,rugamer_semi-structured_2024}.	
	A general SDDR model
    describes every parameter $\theta_i$ of a parametric distribution $\mathcal{D}(\theta_1\dots, \theta_K)$ as $\theta_k=h_k(\eta_k)$, where $h_k:\mathbb{R}\to \Theta_k$ are one-to-one mappings into the right parameter space and where
	\begin{equation}\label{eq: SDDR-modelStruct}
		\eta_k = \textbf{x}^T w+\sum_{j=1}^J f_{k,j}(\textbf{z})+\sum_{l=1}^{L}d_{k,l}(\textbf{u}),
	\end{equation}
    with $\textbf{x}, \textbf{z}$ and $\textbf{u}$ being (potentially overlapping) subsets of input features, $w$ a weight vector, $f_{k,j}$ nonlinear functions formed as linear combinations of some basis-functions (for example B-splines) and $d_{k,l}$ neural networks \citep{rugamer_deepregression_2023}.  

    In our model, $K=1$ and $\mathcal{D}$ is a one-parameter exponential family distribution, parameterized by the canonical parameter $\theta=\theta_1$ (see Equation~\eqref{eq: glm-expFamily}). The predictors of our extension from Equation~\eqref{eq: intro-T1} respectively Equation~\eqref{eq: ldglm-nu1}, can then be seen as a special case of Equation~\eqref{eq: SDDR-modelStruct} with $J=0$ (no structured nonlinear functions), $L=1$, $\textbf{x}=\textbf{u}$ and $d_{1}$ a Lipschitz-constrained neural network.
    Without the extension of Section~\ref{sec: nu2}, the SDDR function $h_1$ corresponds to the function $\Lambda$ in the GLM \eqref{eq: glm-muEquation}.

\section{Monotonicity}\label{sec: append-Monoton}

\begin{lma}\label{lma: nu1-CompMonoton}
	For $q,r\in\mathbb{N}\cup\{\infty\}$, let $T_p:(\mathbb{R}^k, ||.||_q)\to (\mathbb{R}^k, ||.||_r)$ as in Equation \eqref{eq: nu1-deconstructedT1} with $L(\nu_p)<1$. Let for arbitrary $1\leq i \leq k$ be $\pi_i:\mathbb{R}^k\to\mathbb{R}, (x_1,\dots, x_k)\mapsto x_i$ the projection on the $i$-th component and, for arbitrary $\textbf{x}_{-i}=(x_1,\dots, x_{i-1},x_{i+1},\dots, x_k)\in\mathbb{R}^{k-1}$, let $ \iota_{\textbf{x}_{-i}}$ be the natural embedding $ \iota_{\textbf{x}_{-i}}:\mathbb{R}\to\mathbb{R}^k,x_i\mapsto (x_1,, x_{i-1}, x_i,x_{i+1},\dots, x_k)=\textbf{x}$.\\
	Then the function $\pi_i \circ T_p \circ \iota_{\textbf{x}_{-i}}:\mathbb{R}\to\mathbb{R}$ is strictly increasing.
\end{lma}

	\begin{proof}
		We define for easier notation $f := \pi_i\circ T_p\circ \iota_{\textbf{x}_{-i}}$. It is then easy to see that \\ $ f(x_i)= x_i + \pi_i(\nu_p(\iota_{\textbf{x}_{-i}}(x_i)))$. 
		
		In other words, $f= \textbf{Id}_1+ g$
		
		where $\textbf{Id}_1$ is the one-dimensional identity function and $g:= \pi_i\circ \nu_p\circ\iota_{\textbf{x}_{-i}}$.
        Let now $x_i,\widetilde{x_i}\in\mathbb{R}$ be arbitrary. Then:
		\begin{align*}
			|g(x_i)-g(\widetilde{x_i})|&= |\pi_i\circ \nu_p\circ\iota_{\textbf{x}_{-i}}(x_i)-\pi_i\circ \nu_p\circ\iota_{\textbf{x}_{-i}}(\widetilde{x_i})| \overset{(\star)}{\leq} || \nu_p\circ\iota_{\textbf{x}_{-i}}(x_i)-\nu_p\circ\iota_{\textbf{x}_{-i}}(\widetilde{x_i})||_r \\
			\overset{\nu_p\text{ Lipsch.}}&{\leq} L(\nu_p)\cdot ||\iota_{\textbf{x}_{-i}}(x_i)-\iota_{\textbf{x}_{-i}}(\widetilde{x_i})||_q \overset{(\star\star)}{=} L(\nu_p)\cdot |x_i-\widetilde{x_i}|,
		\end{align*} 
		where $(\star)$ holds because for $t=(t_1,\dots, t_k)\in\mathbb{R}^k$:
		$$
		|t_i|^q\leq \sum_{j=1}^{k} |t_j|^q=||t||_q^q \quad \text{for } q\in\mathbb{N} \quad \text{ and } \quad |t_i|\leq \sup_{1\leq j \leq k} |t_j| = ||t||_\infty 
		$$
		and $(\star\star)$ follows since, for any $q\in\mathbb{N}\cup\{\infty\}$, $||t||_q = |t_i|$ if $|t_j|=0$ for all $j\neq i$.\\
		This implies that $g$ is Lipschitz with $L(g)\leq L(\nu_p)<1$ and therefore, for $a<b\in\mathbb{R}$:
		\begin{align*}
			0 \overset{L(g)<1}{<} |b-a| - L(g) \cdot |b-a| \leq b-a - |g(b)-g(a)| \leq b-a + (g(b)-g(a)) = f(b)-f(a),
		\end{align*}
		which is equivalent to $f(b) > f(a)$, proving that $f$ is strictly increasing.
		
	\end{proof}
    
	The above proof cannot easily be extended to cases with $L(\nu_p)\geq 1$, 
    however, as mentioned in Section~\ref{sec: nu1}, we expect cases with $L(\nu_p)< 1$ to be of higher practical interest.
    It is also important to note that generally, for $j\neq i$, the transformed covariate $z_j$ is not necessarily monotone in $x_i$. Even if the latter could be guaranteed, the linear combination with the coefficient vector $\beta$, which might have both positive and negative components, generally prevents monotonicity of $\eta$ in $x_i$, unless restrictive bounds are placed on $L(\nu_p)$.

	\section{Relation to Conditional Normalizing Flows}\label{sec: append-CondNormFlow}
    In this section, we show how Equation~\eqref{eq: intro-T2} fits into the framework defined in \citep{winkler_learning_2023} and could be therefore interpreted as a special case of a conditional NF \citep{algren_flow_2023, winkler_learning_2023}.

	In such a conditional NF with a given input $\textbf{x}\in\mathcal{X}$ and a regression target $y\in\mathcal{Y}$, a mapping $g_\phi:\mathcal{Y}\times\mathcal{X}\to \mathcal{Z}$ is learned that transforms data distributed according to a conditional distribution $f_{\mathcal{Y}|\mathcal{X}}(y|\textbf{x})$ into data distributed according to some known conditional distribution $f_{\mathcal{Z}|\mathcal{X}}(z|\textbf{x})$. $\phi$ denotes the parameters of that mapping. The function $g_\phi$ is bijective in $y$ and $z$ and partially differentiable both directions. In other words, conditional on $\textbf{x}$, 
	\begin{equation}\label{eq: append-CondNormFlow-Winkler}
		Y\overset{d}{=} g_\phi^{-1} (z, \textbf{x})
	\end{equation}
	And, like in any other NF, the density can be evaluated through the transformation formula and new samples can be drawn.
	
	However, in our model, unlike in \citep{winkler_learning_2023}, the target $y$ is not necessarily continuous and the function $h_\theta$ from Equation~\eqref{eq: intro-T2} is not guaranteed to be a diffeomorphism in $y$. In cases where $y$ is in fact continuous and $h_\theta$ does fulfill the relevant conditions of invertibility and differentiability, our model from Equation~\eqref{eq: intro-T2} can be converted into form \eqref{eq: append-CondNormFlow-Winkler} by setting $g_\phi := T_d^{-1} \circ h_\theta ^{-1}$. 
	Compared to \citep{winkler_learning_2023}, this implements the NF in the reverse direction, which is however mostly inconsequential and quite common in the field of NF.\\
	Cases in which $h_\theta$ is not a diffeomorphism are now discussed in more detail. For this, we define, for $u\in\mathbb{R}$, $\tilde{g}_\theta(u):= \sup\{x\in\mathbb{R}: h_\theta(x) \leq u\}$, where analogous arguments as in the proof of Theorem \ref{thm: nu2-inverseTrafoThm} can be used to deduct from the left-continuity and monotonicity of $h_\theta$ that in this definition $\sup$ can be changed into $\max$, and that $\tilde{g}_\theta$ is well-defined and right continuous. We also set $g(u):= T_d^{-1} \circ \tilde{g}_\theta$.
	\paragraph*{Discrete Target Random Variable}
	Similar arguments as in the proof of Theorem \ref{thm: nu2-inverseTrafoThm} show that for $y\in\mathcal{Y}$:
	\begin{align*}
		F(y) &= \mathbb{P}(Y\leq y) = \mathbb{P}(h_\theta(T_d(V))\leq y) = \mathbb{P}(T_d(V)\leq \tilde{g}_\theta( y)) \\
		 &= \mathbb{P}(V\leq T_d^{-1}(\tilde{g}_\theta(y))) = G_0(g(y)).
	\end{align*}
	For $\mathcal{Y}=\{y_i, i\in\mathcal{I}\}$ with $\mathcal{I}\subseteq\mathbb{N}$ with $y_i < y_j$ for all $i<j$, the learned probability mass function can then easily be evaluated as $f(y_i) = F(y_i) - F(y_{i-1})$.
	
	\paragraph*{Continuous Target Random Variable}
	In a traditional (conditional) NF, the density of the transformed random variable is simply computed through the density transformation formula, leveraging the fact that the learned transformation is a diffeomorphism. Through the assumptions made on $h_\theta$ in Theorem \ref{thm: nu2-inverseTrafoThm}, we can guarantee that the inverse of $g$ is uniquely defined and differentiable in $\mathcal{S}^\circ$, meaning we can still use the transformation formula in relevant points. If $f_{V}$ is the density function associated with $G_0$, we therefore define the learned density piece wise as 
	\begin{equation}
		f_\mu(y) = 
		\begin{cases}
			0 \qquad &\quad ,y\in\mathbb{R}\backslash \mathcal{S}\\
			f_V(g^{-1}(y)) (\frac{\partial}{\partial y} g^{-1}(y)) &\quad ,y\in\mathcal{S^\circ} \\
			\lim_{\substack{t\to y \\ t\in\mathcal{S}^\circ}} f_\mu(t) &\quad ,y\in\partial\mathcal{S}\cap \mathcal{S}
		\end{cases}
	\end{equation}
	for $y\in \mathbb{R}$. We deliberately use a slight abuse of notation to define the last case to indicate that this choice of density values on the finite set of border points of $\mathcal{S}$ is essentially arbitrary and that we only choose this definition as we feel it is the most natural continuation of the density to points where it isn't uniquely defined.
	
	\section{Choice of Vector Norms}
    It is a well-known theoretical result that, on finite-dimensional vector spaces, all norms are equivalent. Despite this, we see in the applications of Section~\ref{sec: applications}, a clear tendency by the hyperparameter optimization to select the $||.||_2$-norm over the $||.||_1$-norm. In studying a regularization of NNs through a Lipschitz constraint, \cite{gouk_regularisation_2021} also observe performance differences between norms depending on the application. We do not think that existing results give sufficient indication that certain norms are superior in the context of Lipschitz-constrained NNs, but want to make the reader aware of potential performance differences. We also want to emphasize differences in computational cost and accuracy: general $||.||_{p\to p}$-norms are NP-hard to compute and have to be approximated, whereas the $||.||_{1\to1}$ and $||.||_{\infty\to\infty}$ matrix norms are just the maximum absolute row/column sum \citep{johnston_how_2016}.

\section{Proof of Theorem \ref{thm: nu2-inverseTrafoThm}}\label{append: inverseTrafoProof}
    In this section, we show the proof of Theorem \ref{thm: nu2-inverseTrafoThm}:
    \trafoTheorem*
   \begin{proof}
    We use the notation from Thm.~\ref{thm: nu2-inverseTrafoThm} and prove the Theorem by construction.
	Let $\theta\in \mathcal{I}$ be arbitrary, let $Y\sim f_\theta$ and let $F_\theta:=\mathbb{P}(Y\leq \cdot)$ be the cumulative distribution function (cdf) of $Y$. It is a standard result that $F_\theta$ is non-decreasing with $\lim_{x\to - \infty}F_\theta(x) = 0$ and $\lim_{x\to \infty}F_\theta(x) = 1$. Because of this, the quantile function $Q_{\theta}:(0,1)\to \mathbb{R}$:
	\begin{align*}
		Q_{\theta}(u)& := \inf\{x\in\mathbb{R}:F_\theta(x)\geq u\} \overset{(\star)}{=} \min\{x\in\mathbb{R}:F_\theta(x)\geq u\},
	\end{align*}
    is well defined and because any cdf is right-continuous, $(\star)$ holds. It is standard theory that $Q_{\theta}$ is also monotonically increasing and left-continuous.
    Let now $y\in\mathbb{R}$ be arbitrary. We show that for $u\in(0,1)$, the following holds:
	\begin{equation}\label{eq: nu2-otherFamilies-quantileIneq}
		Q_{\theta}(u)\leq y \iff u\leq F_\theta(y).
	\end{equation}
    The direction ``$\Leftarrow$'' is trivially true as, per definition, $Q_\theta(u)$ fulfils: $Q_\theta(u)\leq x$ for all $x$ with $F_\theta(x)\geq u$. The direction``$\Rightarrow$'' follows because, if $Q_{\theta}(u)\leq y$, then by definition, there exists a $\tilde{y}\leq y$ with $F_\theta(\tilde{y})\geq u$ and by monotonicity: $F_\theta(y)\geq F_\theta(\tilde{y})\geq u$.
    
    Let now $G_0 =\Phi$ be the cumulative distribution function of the standard normal distribution and let $h:\mathcal{I}\times \mathbb{R}\to\mathbb{R}, h(\theta,v):= Q_{\theta}(\Phi(v))$. Then $h$ is increasing in its second argument as a concatenation of monotonous functions and for arbitrary $y\in\mathbb{R}$ and $V\sim G_0$ the following holds:
	\begin{align*}
		\mathbb{P}\big(h_\theta(V)\leq y\big) & \overset{\eqref{eq: nu2-otherFamilies-quantileIneq}}{=} \mathbb{P}(\Phi(V)\leq F_\theta(y)) \overset{G_0(V)\sim  \mathcal{U}[0,1]}{=} F_\theta(y)
	\end{align*}
	implying that  $h_\theta(V)\overset{d}{=}Y$.
    It now only remains to show that $h_\theta$ is also increasing in its second argument. For this, it is sufficient to show that $Q_\theta$ is increasing in $\theta$. 
    We first note that by the regularity assumptions of the exponential family, the canonical parameter $\theta$ is increasing in $\theta$. In \citep{rink_confidence_2025} it is shown that, for a random variable $Y\sim f_\theta$ distributed according to an exponential family distribution, the probability $\mathbb{P}_\theta[Y\geq y ]$ is increasing in $\theta$ for all $y\in\mathbb{R}$. From this, it is then easy to see that $F_\theta(y)$ is decreasing in $\theta$.

    Let now $\theta_1 < \theta_2$ and $u\in(0,1)$ be arbitrary and let $x\in\mathbb{R}$ such that $F_{\theta_2}\geq u$. Then $F_{\theta_1}\geq F_{\theta_2}\geq u$. Therefore, 
    \begin{equation}
        \{x\in\mathbb{R}: F_{\theta_2}(x)\geq u\} \subset \{x\in\mathbb{R}: F_{\theta_1}(x)\geq u\}
    \end{equation}
    and therefore 
    \begin{equation}
        \min\{x\in\mathbb{R}: F_{\theta_2}(x)\geq u\} \geq \min\{x\in\mathbb{R}: F_{\theta_1}(x)\geq u\}
    \end{equation}
    which, by definition, is equivalent to $Q_{\theta_2}(u)\geq Q_{\theta_1}(u) $ and, because $u$ was arbitrary, this proves that $Q_\theta$ is increasing in $\theta$.

\end{proof} 

The choice $G_0 =\Phi$ in the previous proof is of course largely arbitrary. It is in particular not difficult to see that Corollary~\ref{cor: nu2-inverseTrafoThm-referenceDist} holds and that one could simply choose $G_0=F_{\theta_0}$ for a certain reference value of $\theta_0$ if $f_\theta$ are continuous distributions. 

We now show that the probability density of $Y$ from \eqref{eq: intro-T2} can be explicitly computed and that it is compatible with backpropagation.
For this, we consider the cases of discrete and continuous $Y$ separately.
Let first $Y$ be a discrete random variable and use the notation of \eqref{eq: intro-T2} and the proof of Theorem \ref{thm: nu2-inverseTrafoThm}. Then, the following formula holds for any potential value $k$ of $Y$:

\begin{align}
    f_\theta(k) &= \mathbb{P}[Q_\theta \circ \Phi \circ T_d(V) \leq k] - \mathbb{P}[Q_\theta \circ \Phi \circ T_d(V) \leq k] \\
    \overset{\eqref{eq: nu2-otherFamilies-quantileIneq}}&{=} \mathbb{P}[\Phi \circ T_d(V) \leq F_\theta(k)] - \mathbb{P}[\Phi \circ T_d(V) \leq F_\theta(k-1)] \\
    &= \Phi(T_d^{-1}\circ \Phi^{-1} \circ F_\theta(k)) - \Phi(T_d^{-1}\circ \Phi^{-1} \circ F_\theta(k-1)) \label{eq: append-nu2Thm-darstfY-diskret}
\end{align}

\begin{theo}
Let $\omega$ be an arbitrary weight of $T_d$. Then, \eqref{eq: append-nu2Thm-darstfY-diskret} is differentiable in $\omega$ and $\theta$.    
\end{theo}

\begin{proof}
    The proof follows by chain rule. We only prove the Theorem for the first term in \eqref{eq: append-nu2Thm-darstfY-diskret} as the argument for the second term is identical.
    First, the differentiability in $\omega$:
    \begin{align*}
        \frac{\partial}{\partial \omega} \Phi(T_d^{-1}\circ \Phi^{-1} \circ F_\theta(k)) &= \frac{\partial \Phi( T_d^{-1}\circ \Phi^{-1} \circ F_\theta(k))}{(\partial T_d^{-1}\circ \Phi^{-1} \circ F_\theta(k))} \frac{(\partial T_d^{-1}\circ \Phi^{-1} \circ F_\theta(k))}{\partial \omega} \\
        &= \phi(T_d^{-1}\circ \Phi^{-1} \circ F_\theta(k)) \frac{\partial T_d^{-1}}{\partial \omega}(\Phi^{-1} \circ F_\theta(k))
    \end{align*}
    where $\phi$ is the standard normal distribution density. This is well-defined as $\frac{\partial T_d^{-1}}{\partial \omega}$ is well-defined \citep{behrmann_invertible_2019}.
    
    We will not show the entire chain rule computation for differentiability in $\theta$ as it is largely trivial. The crucial part is to show that $F_\theta(k)$ is differentiable in $\theta$.
    However, this follows by Definition~\eqref{eq: glm-expFamily} as in the discrete case $F_\theta(k) = \sum_{l\leq k} \exp\left(\frac{l\theta-b(\theta)}{\phi}+c(l,\phi)\right)$, which is obviously differentiable in $\theta$.
\end{proof}

Let now $Y$ be a continuous random variable with cdf $F_\theta$. As $Y$ is absolutely continuous, its cdf is continuous. To avoid an unnecessarily complex notation, we only consider the case where the support of $Y$ is a single interval $(j_1,j_2)$, which, to our knowledge, includes any practically relevant exponential family distribution. Note that we do not consider the case of (half) closed intervals separately, as the value of the density finitely many points does not matter. Similarly, a generalization of the theory below to distributions with disjoint supports would also be possible, only resulting in $h_\theta$ having jumps on a set of at most countably many points (measure zero). Under these conditions, the following Lemma holds:

\begin{lma}
Let $h_\theta$ be defined as in the proof of Theorem \ref{thm: nu2-inverseTrafoThm}. Then, it is a diffeomorphism from $\mathbb{R}$ to $(j_1,j_2)$.
\end{lma}
\begin{proof}
    Since $\Phi$ is a diffeomorphism $\mathbb{R}\to(0,1)$, it only remains to show that $Q_\theta$ is a diffeomorphism $(0,1)\to(j_1,j_2)$. 
    
    To see that the image of $Q_\theta$ is contained in $(j_1,j_2)$, one can easily conclude from \eqref{eq: nu2-otherFamilies-quantileIneq} that $Q_\theta(u)\leq j_2$ for all $u\in(0,1)$. If for any $u\in(0,1)$ it was the case that $Q_\theta(u)\leq j_1$, then \eqref{eq: nu2-otherFamilies-quantileIneq} implies that $0<u\leq F_\theta(j_1) = 0$, which would be a contradiction. Therefore, the image of $Q_\theta$ is a subset of $(j_1,j_2)$.

    On $(j_1,j_2)$ however, $F_\theta$ is, by definition, invertible and it is a well known result that the inverse function is then the quantile function. Also, because $F_\theta$ is differentiable on the entire $(j_1,j_2)$, the inverse function theorem implies that $Q_\theta$ is also differentiable on the entire $(0,1)$ and therefore a diffeomorphism.
\end{proof}

The previous Lemma implies that $h_\theta$ is compatible with the density transformation theorem and therefore the density of \eqref{eq: intro-T2} can be computed explicitly.
We now also show that $h_\theta$ is compatible with gradient descent, i.e.\ that it is also differentiable with regards to the weights of $T_p$. For this, we only need to prove differentiability in $\theta$:
\begin{lma}
    Let $h_\theta$ be defined as in the proof of Theorem \ref{thm: nu2-inverseTrafoThm}. Then, the partial derivative $ \frac{\partial h_\theta}{\partial \theta}(u)$ exist for all $u\in\mathbb{R}$ and $\frac{\partial h^{-1}}{\partial \theta}(y)$ exist for all $y\in (j_1,j_2)$.
\end{lma}
\begin{proof}
    We interpret $h_\theta$ as a function $h:\mathbb{R}^2\to\mathbb{R}^2, (u,\theta)\mapsto (h_\theta(u), \theta)=\big(Q_\theta(\Phi(u)), \theta\big)$.
    It is easy to see that this 2-dimensional function is invertible with $h^{-1}(y,\theta)=(h_\theta^{-1}(y), \theta) = \big(\Phi^{-1}(F_\theta(y)),\theta \big)$. From this, one can easily compute the following via the chain rule:
    \begin{equation}
        \frac{\partial h_\theta^{-1}}{\partial y} = \frac{\partial \Phi^{-1}(F_\theta(y)))}{\partial F_\theta} \frac{\partial F_\theta}{\partial y} = \frac{\partial \Phi^{-1}(F_\theta(y))}{\partial F_\theta} f(y)
    \end{equation}
    and 
    \begin{equation}
        \frac{\partial h_\theta^{-1}}{\partial \theta} = \frac{\partial \Phi^{-1}(F_\theta(y)))}{\partial F_\theta} \frac{\partial F_\theta}{\partial \theta},
    \end{equation}
    where the existence of $\frac{\partial F_\theta}{\partial \theta }$ can be easily concluded from the definition of exponential families in \eqref{eq: glm-expFamily} and common regularity assumptions of GLM.
    Therefore, $Dh^{-1} = \begin{pmatrix}
        \frac{\partial \Phi^{-1}(F_\theta(y))}{\partial F_\theta} f(y) & \frac{\partial \Phi^{-1}(F_\theta(y)))}{\partial F_\theta} \frac{\partial F_\theta}{\partial \theta} \\
        0 & 1
    \end{pmatrix}.$
     As the determinant $Dh^{-1}$ is simply $f_\theta$ times a strictly positive constant, this matrix is invertible for all $\theta\in\mathbb{R}$ and $y\in(j_1,j_2)$. The inverse function theorem then implies that $h^{-1}_\theta$ and its inverse are partially differentiable in both arguments.
\end{proof}

\section{Details on PHO}\label{append: proofOrthog}
We start our discussion with an adaptation of \citep[Def. 2.1.]{rugamer_semi-structured_2024}:
 
\begin{defi}\label{def: ldglm-orthog-defIdentif}
	Let $((\beta_0, \beta), W)$ denote the parameters in \eqref{eq: intro-T1}, where $W$ denotes the weights of $\nu_p$; for clarity we write $\nu_p^W$.
	Let $\textbf{X}\in \mathbb{R}^{n\times k}$ be the design matrix containing the training data; for a simpler notation we also write $ \nu_p^{W}(\textbf{X})\in \mathbb{R}^{n\times k}$ for the matrix containing the NN's output on the training data.
    
	We then call a model of the form \eqref{eq: ldglm-nu1} identifiable together with a loss-function $l((\beta_0, \beta),W)$, if there exist no two configurations $((\beta_0, \beta),W)$ and $ ((\delta_0, \delta),V)$ of the model's parameters such that $l((\beta_0, \beta),W)=l((\delta_0, \delta),V)$ and for which:
	\begin{equation}\label{eq: ldglm-orthog-defIdentif-eq}
	\beta_0+\textbf{X}\beta+\nu_p^{W}(\textbf{X})\beta= \delta_0 + \textbf{X}\delta+\nu_p^{V}(\textbf{X})\delta
	\end{equation}
	and simultaneously
	\begin{equation}\label{eq: ldglm-orthog-defIdentif-ineq}
		\beta_0+ \textbf{X}\beta\neq\delta_0 + \textbf{X}\delta    \qquad \text{ or } \qquad \nu_p^{W}(\textbf{X})\beta\neq\nu_p^{V}(\textbf{X})\delta.
	\end{equation}

\end{defi}

As we show later, PHO guarantees this property without changing the model's predictions (neither predicted mean nor distribution) but performs a single orthogonalization step after training to improve interpretability of the model's terms.

\subsection{Computation of PHO}

As PHO is a post-hoc-method, we assume to be given a fitted model as in Equation~\eqref{eq: intro-T1} and also use the notation of \eqref{eq: ldglm-nu1}.
To extract all linear effects contained in $\nu_p$, we perform a multivariate linear regression of the $(n \times k)$-matrix $\nu_p(\textbf{X})$ w.r.t. the original design matrix with an added constant column (to also orthogonalize w.r.t. the intercept) $(\textbf{1}|\textbf{X})$. The least squares solution of such a model is then \citep{christensen_linear_1991} given by:
\begin{equation}
	\widehat{\nu_p(\textbf{X})} = (\textbf{1}|\textbf{X}) \hat{\Gamma} = M \cdot \nu_p(\textbf{X}),
\end{equation}
where $\hat{\Gamma} = \begin{pmatrix}
	\gamma_{11} & \dots & \gamma_{1k} \\
	\vdots & \ddots & \vdots \\
	\gamma_{(k+1)1} & \dots & \gamma_{(k+1)k}
\end{pmatrix} = ((\textbf{1}|\textbf{X})^T (\textbf{1}|\textbf{X}))^{-1} (\textbf{1}|\textbf{X})^T \nu_p(\textbf{X})$ \citep{fahrmeir_multivariate_1983}, and where $M$ is the orthogonal projection on the column space of $\textbf{X}$, i.e. $M=\mathcal{P}_\textbf{X}$. For arbitrary $\textbf{x}\in\mathbb{R}^k$, the transformation $T_p$ from \eqref{eq: intro-T1} can then be rewritten as:
\begin{equation}\label{eq: ldglm-orthog-defNuBreve}
	T_p(x) = x + \nu_p(x) = x + (1|x) \hat{\Gamma} + \underbrace{(\nu_p(x)- (1|x)\hat{\Gamma})}_{=:\breve{\nu_p}(x)}.
\end{equation}
Writing then $\hat{\Gamma}_{-B}$ for $\hat{\Gamma}$ without its first row, we can re-write the computation of the predictor $\eta$ as: 
\begin{align}\label{eq: ldglm-orthog-defBetaTilde}
	\eta(\textbf{x}) \overset{\eqref{eq: ldglm-orthog-defNuBreve}}{=}\beta_0 +  T_p(\textbf{x})\beta= \underbrace{\Bigg(\beta_0 + \sum_{i=1}^{k} (\gamma_{1i} \beta_i) \Bigg)}_{=:\widetilde{\beta_0}} + \textbf{x}\cdot\underbrace{\Bigg((\text{Id} + \hat{\Gamma}_{-B}) \cdot\beta\Bigg)}_{=:\widetilde{\beta}} + \breve{\nu_p}(\textbf{x})\beta.
\end{align}

In order to get back a model of the form \eqref{eq: ldglm-nu1}, we can define the diagonal matrix\\ $D:= \text{diag}(\beta_i/\widetilde{\beta}_i;\; i=1,\dots, k)$ and set $\widetilde{\nu_p}(\textbf{x}) := \breve{\nu_p}(\textbf{x}) \cdot D$, implying: 
\begin{equation}\label{eq: ldglm-orthog-orthogModel}
	\eta = \widetilde{\beta_0} + (\textbf{x} + \widetilde{\nu_p}(\textbf{x}))\widetilde{\beta}.
\end{equation}
It can now be easily seen that each column of $\breve{\nu_p}(\textbf{X})$ is orthogonal to the span of columns of \textbf{X}. This follows because, as is discussed above, $(\textbf{1}|\textbf{X}) \hat{\Gamma} = \mathcal{P}_\textbf{X}(\nu_p(\textbf{X}))$ and therefore $\breve{\nu_p}(\textbf{X})= \nu_p(\textbf{X})- (1|\textbf{X})\hat{\Gamma} = \mathcal{P}^\perp_\textbf{X}(\nu_p(\textbf{X}))$. This orthogonality is then obviously also fulfilled by $\widetilde{\nu_p}(X)$, in which the columns are just scaled by one scalar each compared to $\breve{\nu_p}(X)$. From this orthogonality, the following Lemma follows almost immediately:

\begin{lma}\label{lma: orthog-identif}
	The model \eqref{eq: ldglm-orthog-orthogModel} is identifiable. 
\end{lma}
\begin{proof}
	Assume there exist weights $ (\widetilde{\beta_0},\widetilde{\beta}, W)$ and $(\widetilde{\delta_0}, \widetilde{\delta}, V)$ such that for these weights, model \eqref{eq: ldglm-orthog-orthogModel} fulfills both Equation~\eqref{eq: ldglm-orthog-defIdentif-eq} and one of the inequalities \eqref{eq: ldglm-orthog-defIdentif-ineq}, which immediately implies it fulfils both inequalities. Then, rearranging Equation~\eqref{eq: ldglm-orthog-defIdentif-eq} results in:
	\begin{equation}
		\underbrace{\widetilde{\nu_p}^W \widetilde{\beta} - \widetilde{\nu_p}^V \widetilde{\delta}}_{\in \text{ span}((\textbf{1}|\textbf{X}))^\perp} = \underbrace{\delta_0 + \textbf{X}\widetilde{\delta} - \widetilde{\beta_0}+ \textbf{X}\widetilde{\beta}}_{\in \text{ span}((\textbf{1}|\textbf{X}))},
	\end{equation}
	where the right side is an element of the span of the columns of $(\textbf{1}|\textbf{X})$ and the left side, as is discussed above, is an element of the orthogonal complement (as a linear combination of such elements). This implies that both sides must equal the zero vector, which contradicts the assumed inequalities \eqref{eq: ldglm-orthog-defIdentif-ineq}.

\end{proof}

Compared to \citep{rugamer_new_2023}, our PHO-formulation lays a larger focus on the multidimensional output of $\nu_p$ rather than the predictors computed from the structured and unstructured parts respectively as, unlike in standard SDDR models, the individual outputs of $\nu_p$ potentially have a meaningful interpretation as dimension-wise nonlinear correction terms. The actual weights of $\widetilde{\beta}$ are however equal to those computed by \cite{rugamer_new_2023} and we could in particular also make use of the efficient exact computation methods discussed in \citep[Alg. 2]{rugamer_new_2023} if necessary.

\section{Details on Applications}
    \subsection{Nested Cross Validation Details}\label{append: appl-nestedCross}
    For performance comparison in Section~\ref{sec: applications}, we used a nested cross validation approach with a 5-fold split in the outer loop and a simple holdout approach in the inner loop. 
    For each model, except for the GLM, we performed a grid search hyperparameter optimization in the inner loop and evaluated the optimal configuration of each method on the test data of the outer loop.
For models that benefit from early stopping, we do not refit on the combined training/validation data before testing in the outer loop, but instead use the validation data for early stopping to emulate a real application.
	\subsection{Cars Application}\label{append: appl-detailsCar}

	In the cars data application, two different LiD-GLM variants were tested. One including a distributional correction $T_d$ and one without such a correction. Both were based on a Gaussian distribution and an identity link function and were trained for a maximum of 3000 epochs with early stopping with a patience of 200 epochs. If $T_p$ and $T_d$ were trained one at a time, then the epochs were split evenly and training of $T_d$ was performed with a patience of 100. After training, we reloaded the best model parameters (as measured on validation data). We used the negative log-likelihood as a loss criterion and, as the dataset is quite small, only used one batch.

    The hyperparameter grids for the LiD-GLM variants are listed in the main paper and in Supplement~\ref{suppl: addRes-cars}. That for SDDR models is given in Table~\ref{tab: appl-cars-hyperpNN}.

    \begin{table}[t]
		\centering
		\small
		\tabcolsep=0.11cm
		\begin{tabular}{r|r|r|r|r|r|r}
			\toprule
			  learning rate & width & depth & width 2 & depth 2 & number batches & dropout rate\\
            \midrule
            {\renewcommand{\arraystretch}{0.5}\begin{tabular}{@{}r@{}} $10^{-4}$, $10^{-3}$ \\ $10^{-2}$, $10^{-1}$ \end{tabular}} &
            {\renewcommand{\arraystretch}{0.5}\begin{tabular}{@{}r@{}} 9, 18, \\36 \end{tabular}} &
            {\renewcommand{\arraystretch}{0.5}\begin{tabular}{@{}r@{}} 2, 3, \\4 \end{tabular}} &
            {\renewcommand{\arraystretch}{0.5}\begin{tabular}{@{}r@{}} 9, 18, \\36 \end{tabular}} &
            {\renewcommand{\arraystretch}{0.5}\begin{tabular}{@{}r@{}} 2, 3, \\4 \end{tabular}} &
            1, 10 & 0, 0.1, 0.2\\
			\bottomrule
		\end{tabular}
		\captionof{table}{Tested hyperparameters for the SDDR model on the ``Cars''dataset. Used was a Normal distribution family and a maximum of 3000 epochs.}
        \label{tab: appl-cars-hyperpNN}
	\end{table}

\subsection{Insurance Application}\label{append: appl-detailsInsurance}

The ``Insurance'' dataset \citep{lantz_machine_2019} contains synthetic data on the total medical expenses of 1338 patients for a given calendar year. The independent covariates are ``age'' (of the primary beneficiary), ``sex'' (male or female), ``bmi'' (body mass index), `children'' (number of children covered by the insurance plan), ``smoker'' (whether that client regularly smokes tobacco) and ``region'' (client's region in the US (northeast, southeast, southwest, or northwest)). In \citep{lantz_machine_2019} the authors describe the data as having been created based on US Census Bureau statistics.
In Table~\ref{tab: appl-insurance-hyperpSDDR} we list the tested hyperparameters for the SDDR model in this example.

\begin{table}[t]
		\centering
		\small
		\tabcolsep=0.11cm
		\begin{tabular}{r|r|r|r|r|r|r}
			\toprule
			  learning rate & width & depth & width 2 & depth 2 & number batches & dropout rate\\
            \midrule
            {\renewcommand{\arraystretch}{0.5}\begin{tabular}{@{}r@{}} $10^{-4}$, $10^{-3}$ \\ $10^{-2}$, $10^{-1}$ \end{tabular}} &
            {\renewcommand{\arraystretch}{0.5}\begin{tabular}{@{}r@{}} 9, 18, \\36 \end{tabular}} &
            {\renewcommand{\arraystretch}{0.5}\begin{tabular}{@{}r@{}} 2, 3, \\4 \end{tabular}} &
            {\renewcommand{\arraystretch}{0.5}\begin{tabular}{@{}r@{}} 9, 18, \\36 \end{tabular}} &
            {\renewcommand{\arraystretch}{0.5}\begin{tabular}{@{}r@{}} 2, 3, \\4 \end{tabular}} &
            1, 10 & 0, 0.1, 0.2\\
			\bottomrule
		\end{tabular}
		\captionof{table}{Tested hyperparameters for the SDDR model on the ``Insurance''dataset. Used was a Normal distribution family and a maximum of 3000 epochs.}
        \label{tab: appl-insurance-hyperpSDDR}
	\end{table}

	\subsection{Stroke Application}\label{append: appl-detailsStroke}
	
	In the stroke application, the LiD-GLMs used a  Bernoulli distribution and a logit link function and were trained for a maximum of 3000 epochs with early stopping with a patience of 500 epochs. As the datasest is quite large, we used 10 minibatches for training. Again, we used the negative log-likelihood as a loss criterion and after training, we reloaded the best model parameters as measured on validation data.
    The tested hyperparameters for LiD-GLM were listed in the main paper, those for NNs are listed in Table~\ref{tab: appl-stroke-hyperpNN} and those for SDDR models in Table~\ref{tab: appl-stroke-hyperpSDDR}.

    \begin{table}[t]
		\centering
		\small
		\tabcolsep=0.11cm
		\begin{tabular}{r|r|r|r}
			\toprule
			  learning rate & width & depth & number of batches\\
            \midrule
            {\renewcommand{\arraystretch}{0.5}\begin{tabular}{@{}r@{}} $10^{-5}$, $10^{-4}$ \\ $10^{-3}$ \end{tabular}} &
            {\renewcommand{\arraystretch}{0.5}\begin{tabular}{@{}r@{}} 8, 16, 32,\\ 64, 100 \end{tabular}} &
            {\renewcommand{\arraystretch}{0.5}\begin{tabular}{@{}r@{}} 2, 3, 4 \\ 5, 6 \end{tabular}} &
            1, 10\\
			\bottomrule
		\end{tabular}
		\captionof{table}{Tested hyperparameters for the neural network on the ``Stroke''dataset. Used was a maximum of 5000 epochs.}
        \label{tab: appl-stroke-hyperpNN}
	\end{table}

    \begin{table}[t]
		\centering
		\small
		\tabcolsep=0.11cm
		\begin{tabular}{r|r|r|r|r}
			\toprule
			  learning rate & width & depth & weight decay & dropout rate\\
            \midrule
            {\renewcommand{\arraystretch}{0.5}\begin{tabular}{@{}r@{}} $10^{-4}$, $10^{-3}$ \\ $10^{-2}$ \end{tabular}} &
            {\renewcommand{\arraystretch}{0.5}\begin{tabular}{@{}r@{}} 36, 64,\\ 100 \end{tabular}} &
            {\renewcommand{\arraystretch}{0.5}\begin{tabular}{@{}r@{}} 3, 4 \\ 5, 6 \end{tabular}} &
            0, 0.01 & 0, 0.1\\
			\bottomrule
		\end{tabular}
		\captionof{table}{Tested hyperparameters for the SDDR model on the ``Stroke''dataset. Used was a Bernoulli family, a maximum of 3000 epochs with a batch size of 3000 with a patience of 200, an output shape of one and zero degrees of freedom.}
        \label{tab: appl-stroke-hyperpSDDR}
	\end{table}

\section{Additional Results}
In this section we show additional results for the applications in Section~\ref{sec: applications} that were not already listed in the main paper.

\subsection{Cars Application}\label{suppl: addRes-cars}

Table~\ref{tab: cars-topConfigs_with_t_d} shows the optimal hyperparameters found in each split for LiD-GLMs with $T_d$ on the nested cros validation described in Section~\ref{sec: applications-cars-perform}.
Like in the case without $T_d$ the results again indicate a preference for small networks $T_p$ of only 9 nodes in 2-3 hidden layers. Compared to Table~\ref{tab: cars-topConfigs_no_T_d}, the preferred Lipschitz bounds $L_p^b$ were slightly larger in this comparison and largest possible bounds were generally preferred. The preferred Lipschitz bounds for $T_d$ show a wide range, but larger constants show a slight prevalence. Together, this could explain the high variance of top model performance in the case of LiD-GLM with $T_d$. 

\begin{table}
	\centering
	\small
	\tabcolsep=0.11cm
	\begin{tabular}{c|r|r|r|r|r|r|r|r|r|r|r}
		\toprule
		$k$ & $L_p^b$ & wid. $T_p$ & dep. $T_p$ & $L_d^b$ & wid. $T_d$ & dep. $T_d$ & blk. $T_d$ & act. $T_d$ & Training & test NLL\\
		\midrule
        \midrule
        $\mathcal{H}$ & 
        {\renewcommand{\arraystretch}{0.5}\begin{tabular}{@{}r@{}} 0.01- \\ 1.0 \end{tabular}} &
        9, 18 & 2, 3 &
        {\renewcommand{\arraystretch}{0.5}\begin{tabular}{@{}r@{}} 0.4, 0.6 \\ 0.99 \end{tabular}} &
        4, 8 & 2, 3 & 1,2 &
        {\renewcommand{\arraystretch}{0.5}\begin{tabular}{@{}r@{}} Swish, \\ ReLU \end{tabular}} & 
        {\renewcommand{\arraystretch}{0.5}\begin{tabular}{@{}r@{}} full, \\ ``CW'' \end{tabular}} &\\
        \midrule
        \midrule
        
		0  & 1.00 & 9  & 2 & 0.99 & 4 & 3 & 1 & Swish & full & 0.254 \\
		1  & 1.00 & 9  & 3 & 0.99 & 4 & 3 & 2 & Swish & full & 0.269 \\
		2  & 1.00 & 9  & 2 & 0.99 & 8 & 3 & 2 & Swish & full & 0.599\\
		3  & 0.88 & 9  & 2 & 0.40 & 4 & 3 & 2 & ReLU & ``CW'' & 0.429\\
		4  & 1.00 & 9  & 2 & 0.60 & 4 & 2 & 1 & Swish & full & 0.551 \\

		\bottomrule
	\end{tabular}
	\captionof{table}{The best hyperparameters found for LiD-GLM with $T_d$ on each split $k$ in the cars data example. Used were $||.||_2$, with one residual block and the ReLU activation used for $T_p$ and a learning rate of $10^{-3}$ for 3000 epochs with early stopping (patience 200) using the negative loglikelihood loss training either the entire model at once ("full"), or first the mean prediction and then $T_d$ (component-wise, "CW").}
	\label{tab: cars-topConfigs_with_t_d}
\end{table}

\subsection{Insurance Application}\label{suppl: addRes-insurance}

\begin{table}[t]
	\centering
	\small
	\tabcolsep=0.11cm
	\begin{tabular}{c|r|r|r|r|r|r|r|r|r|r|r}
		\toprule
        $k$ & $||.||_q$ & $L_p^b$ & lr ($10^x$) & wid. $T_p$& dep. $T_p$ &blk. $T_p$ &act. $T_p$  & $\beta$ frozen & test NLL\\
		\midrule
        \midrule
        $\mathcal{H}$ & 
        1, 2 &
        {\renewcommand{\arraystretch}{0.5}\begin{tabular}{@{}r@{}} 0.01- \\ 1.0 \end{tabular}} &
        -5, -3 & 
        {\renewcommand{\arraystretch}{0.5}\begin{tabular}{@{}r@{}} 9, 18, \\ 36 \end{tabular}} &
        {\renewcommand{\arraystretch}{0.5}\begin{tabular}{@{}r@{}} 1, 2, \\ 3 \end{tabular} }&
        {\renewcommand{\arraystretch}{0.5}\begin{tabular}{@{}r@{}} 1, 2\end{tabular} }&
        {\renewcommand{\arraystretch}{0.5}\begin{tabular}{@{}r@{}} GroupSort, \\ ReLU \end{tabular} }&
        {\renewcommand{\arraystretch}{0.5}\begin{tabular}{@{}r@{}} yes, \\ no \end{tabular} }& \\   
        \midrule
        \midrule
		
		0  & 1 & 0.85 & $-3$ & 18 & 1 & 1 & ReLU & no &  0.446\\
		1  & 1 & 0.85 & $-3$ & 9  & 1 & 1 & GroupSort & no & 0.438 \\
		2  & 2 & 0.95 & $-3$ & 9  & 3 & 1 & GroupSort & no & 0.495 \\
		3  & 2 & 0.90 & $-3$ & 9  & 2 & 1 & ReLU & no &  0.482 \\
		4  & 2 & 0.55 & $-3$ & 9  & 2 & 2 & GroupSort & no & 0.533\\

		\bottomrule
	\end{tabular}
	\captionof{table}{Optimal hyperparameters for the ``Insurance'' data found for LiD-GLM without $T_d$ on each split $k$. Used were, 3000 epochs with a patience of 200, 1 batch and the negative loglikelihood loss.}
	\label{tab: insurance-topConfigs_no_T_d}
\end{table}

\begin{table}
	\centering
	\small
	\tabcolsep=0.11cm
	\begin{tabular}{c|r|r|r|r|r|r|r|r|r|r|r}
		\toprule
		$k$ & $L_p^b$ & wid. $T_p$ &  dep. $T_p$ &$L_d^b$ & wid. $T_d$ & dep. $T_d$ & blk. $T_d$ & act. $T_d$ & Training & test NLL\\
		\midrule
        \midrule
        $\mathcal{H}$ & 
        {\renewcommand{\arraystretch}{0.5}\begin{tabular}{@{}r@{}} 0.01- \\ 1.0 \end{tabular}} &
        9, 18 & 2, 3 &
        {\renewcommand{\arraystretch}{0.5}\begin{tabular}{@{}r@{}} 0.4, 0.99, \\ 2.99 \end{tabular}} &
        6, 16 & 2, 3 & 1, 2 &
        {\renewcommand{\arraystretch}{0.5}\begin{tabular}{@{}r@{}} Swish, \\ ReLU \end{tabular}} &
        {\renewcommand{\arraystretch}{0.5}\begin{tabular}{@{}r@{}} full, \\ "CW" \end{tabular}} &\\
        \midrule
        \midrule
        
		0  & 0.63  & 18  & 2 & 2.99 & 6  & 3 & 2 & Swish & full & -1.184 \\
		1  & 1.0   & 18  & 3 & 2.99 & 6  & 3 & 2 & Swish & full & -1.180 \\
		2  & 0.63  & 9   & 3 & 2.99 & 6  & 3 & 2 & Swish & full & -0.865\\
		3  & 0.88  & 9   & 3 & 2.99 & 16 & 3 & 2 & Swish & full & -1.224\\
		4  & 0.63  & 9   & 2 & 2.99 & 6  & 3 & 2 & Swish & full & -0.672\\

		\bottomrule
	\end{tabular}
	\captionof{table}{The best hyperparameters found for LiD-GLM with $T_d$ on each split $k$ in the ``Insurance'' data example. Used were $||.||_2$, lr = $10^{-3}$, 3000 epochs with a patience of 200 and 1 batch, 1 residual block in $T_p$ and the ReLU activation in $T_p$. After initialization, either all parameters were given to the optimizer at once ("full"), or first $\beta$ and $T_p$ and afterwards $T_d$ were trained (component-wise, "CW").}
	\label{tab: insurance-topConfigs_with_t_d}
\end{table}

We shown in Table~\ref{tab: insurance-topConfigs_no_T_d} the best hyperparameters for LiD-GLM without distributional correction in each split in the nested crossvalidation from Section~\ref{sec: applications-insurance-perform}, in Table~\ref{tab: insurance-topConfigs_with_t_d} those with distributional correction and in Table~\ref{tab: appl-insurance-coeffTable-notFrozen} the coefficient table for the model from Section~\ref{sec: applications-insurance-interpr}.
One can see from Table~\ref{tab: insurance-topConfigs_with_t_d}, that the largest possible $L_d^b$ was chosen in every split which likely indicates a significant deviation of the data from the assumption of Normal residuals. The table also indicates that in this application narrow and deep networks were preferred to wide shallow ones and ``Swish'' was clearly preferred over ReLU as an activation function in $T_d$.

It is also interesting to note that for the model from Section~\ref{sec: applications-insurance-interpr}, we used the ExLipBaB algorithm \citep{splittgerber_exlipbab_2026} to compute the exact Lipschitz constant, which resulted in $L(\nu_p)=0.984$, meaning that in the insurance application the allowed Lipschitz bound was almost fully utilized by the optimizer unlike in the cars data example.

\begin{table}[t]
\centering
	\begin{tabular}{lrrr}
		\toprule
        Variable & GLM & LiD-GLM &$R^2$\\
        \midrule
        bias & -1.23 & -1.45 & - \\
        age  & 0.31 & 0.31 & 0.97 \\
        sex & 0.00 & 0.01 & 1.00 \\
        bmi & 0.03 & 0.03 &  1.00 \\
        children & 0.05 & 0.05 & 1.00\\
        smoker & 1.95 & 1.82 & 0.92 \\
        region northwest & -0.07 & -0.02 & 0.77 \\
        region southeast & -0.13 & -0.05 & 0.92\\
        region southwest & -0.09 & -0.04 & 0.90\\
        \bottomrule

    \end{tabular}
    \caption{Coefficient table for LiD-GLM trained on the Stroke data. Shown are the coefficients of the Linear Model and the LiD-GLM (including PHO) and the coefficients of determination $R^2$ for the variables $\beta_0$=Intercept, $\beta_1$=Sex, $\beta_2$=Age, $\beta_3$=Hypertension, $\beta_4$=Heart disease, $\beta_5$=ever married, $\beta_6$=Residence type, $\beta_7$= Average glucose level, $\beta_8$=BMI, $\beta_9$=smoking status.}
	\label{tab: appl-insurance-coeffTable-notFrozen}
\end{table}

\subsection{Stroke Application}\label{suppl: addRes-stroke}
In Table~\ref{tab: appl-stroke-coeffTable-notFrozen} we show the coefficient table for a LiD-GLM with the same hyperparameters as in Section~\ref{sec: applications-stroke-interpr} but with linear coefficients $\beta$ that aren't frozen during training of the NNs. As is to be expected, this model hat an improved NLL of 0.116. However, the linear coefficients were also changed to a much larger degree than in the ``Cars'' application. The new $R^2_i$-values should also be interpreted in the context of these changed linear coefficients, meaning that even large $R^2_i$ only certify similarity to these new changed linear coefficients. Also, the $R_i^2$ for the ``heart disease'' covariate is drastically larger than for the model with frozen $\beta$. We also show corresponding PDP plots in Figures \ref{fig: appl-stroke-pdp-notFrozen} and \ref{fig: appl-stroke-pdpTransfVars-notFrozen}. The PDP plots of the numeric covariates in Figure~\ref{fig: appl-stroke-pdp-notFrozen} look almost identical to those in Figure~\ref{fig: appl-stroke-pdp}, just on a different scale. The PDPs for ``hypertension'', ``heart disease'' and ``ever married'' in Figures~\ref{fig: appl-stroke-pdpTransfVars-notFrozen} and \ref{fig: appl-stroke-pdpTransfVars} are also almost identical (modulo a change in scale), but the other PDPs appear to show significant differences.
\begin{table}[t]
\centering
	\begin{tabular}{lrrrrrrrrrr}
		\toprule
        & $\beta_0$ & $\beta_1$ &  $\beta_2$ & $\beta_3$ & $\beta_4$ & $\beta_5$ & $\beta_6$ & $\beta_7$ & $\beta_8$ & $\beta_9$ \\
        \midrule
	       LM & -1.0 & -0.41 & 0.09 & 1.23 & 1.13 & 0.87 & 0.13 & 0.42 & -0.16 & 0.13 \\
	       LiD-GLM & -7.83 & -1.38 & 0.25 & 4.09 & 3.76 & 5.33 & 1.32 & 1.28 & -1.6 & -0.19 \\
	       $R^2_i$ & - & 0.96 & 0.99 & 0.52 & 0.09 & 0.93 & 0.98 & 0.99 & 1.0 & 0.86 \\
        \bottomrule

    \end{tabular}
    \caption{Coefficient table for LiD-GLM trained on the Stroke data. Shown are the coefficients of the Linear Model and the LiD-GLM (including PHO) and the coefficients of determination $R^2$ for the variables $\beta_0$=Intercept, $\beta_1$=Sex, $\beta_2$=Age, $\beta_3$=Hypertension, $\beta_4$=Heart disease, $\beta_5$=ever married, $\beta_6$=Residence type, $\beta_7$= Average glucose level, $\beta_8$=BMI, $\beta_9$=smoking status.}
	\label{tab: appl-stroke-coeffTable-notFrozen}
\end{table}

\begin{figure}
	\centering
	\includegraphics[width=\textwidth]{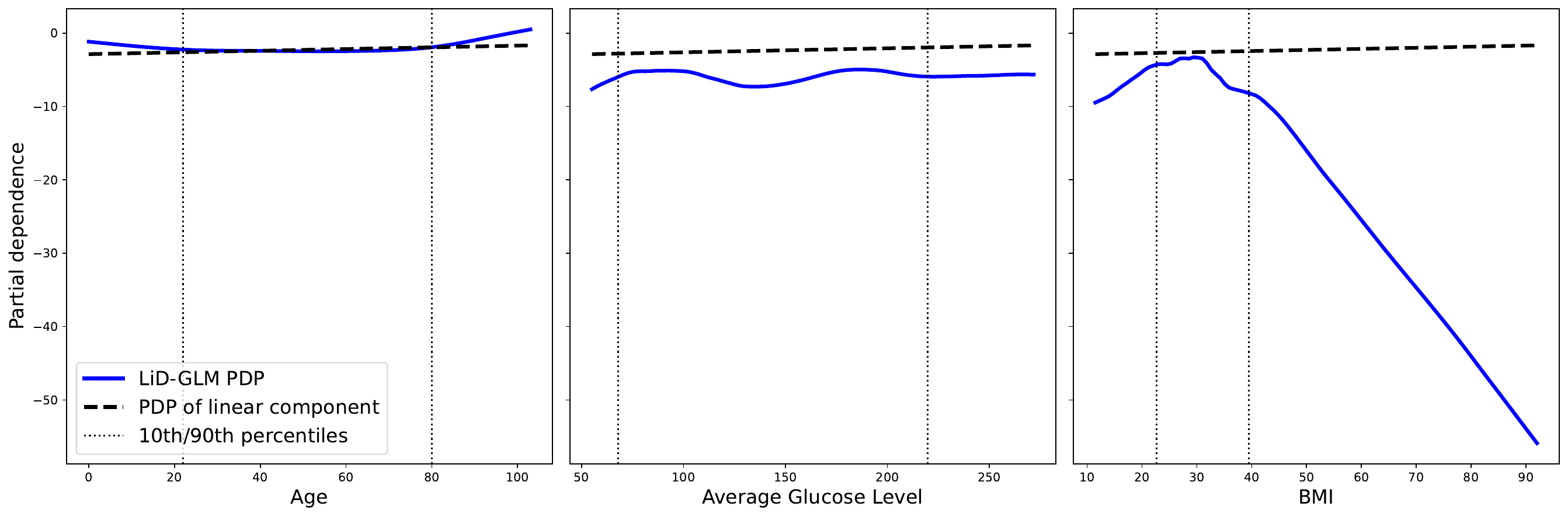}
	\caption{PDP plots of LiD-GLM on the continuous variables in stroke data example w.r.t. predicted logits.}
	\label{fig: appl-stroke-pdp-notFrozen}
\end{figure}

\begin{figure}[t]
	\centering
	\includegraphics[width=0.9\textwidth]{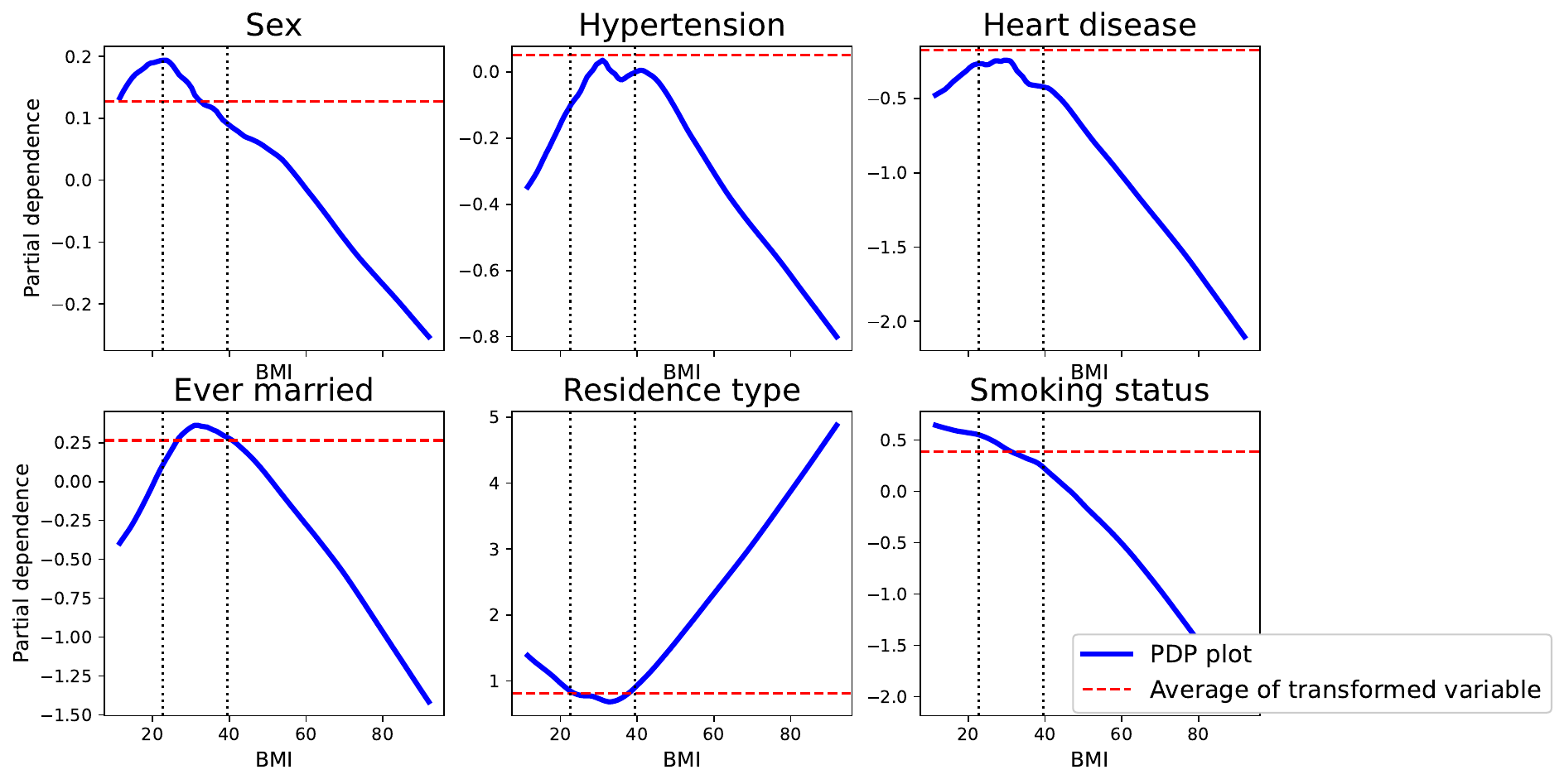}
	\caption{PDP plots of transformation $T_p$ w.r.t.\ ``BMI'' of the former binary covariates.}
    \label{fig: appl-stroke-pdpTransfVars-notFrozen}
\end{figure}

\subsection{PHO}\label{suppl: addRes-Pho}
In Figure~\ref{fig: appl-cars_perf-PHO_vis_with_nu_d}, we show the effect of PHO for LiD-GLM trained on the cars data with a distributional correction $T_d$ and in Figure \ref{fig: appl-cars_perf-PHO_vis_no_nu_d} we show the results for all splits for LiD-GLM without $T_d$. As can be easily seen, the reduction in variance for the estimated linear coefficients is almost the same in both cases.

The corresponding plots for LiD-GLM on the ``Insurance'' dataset are shown in Figures~\ref{fig: appl-insurance_perf-PHO_vis_with_nu_d} and \ref{fig: appl-insurance_perf-PHO_vis_no_nu_d}. As can be seen, PHO again results in a reduced variance of estimated linear effects in both cases.

In Figure \ref{fig: appl-stroke_perf-PHO_vis_all_splits}, we show corresponding results for the ``Stroke'' application. This is a less relevant setting for LiD-GLMs as a good performing model on this dataset requires a relatively large allowed Lipschitz constant and therefore a drastically reduced interpretability. However, PHO still reduces the variance and absolute values in the parameter estimates for four covariates ``sex'', ``age'', ``hypertension'' and ``heart disease'', but seems to have no such effect for the other five covariates.

\begin{figure}[t]
    \includegraphics[width=\textwidth]{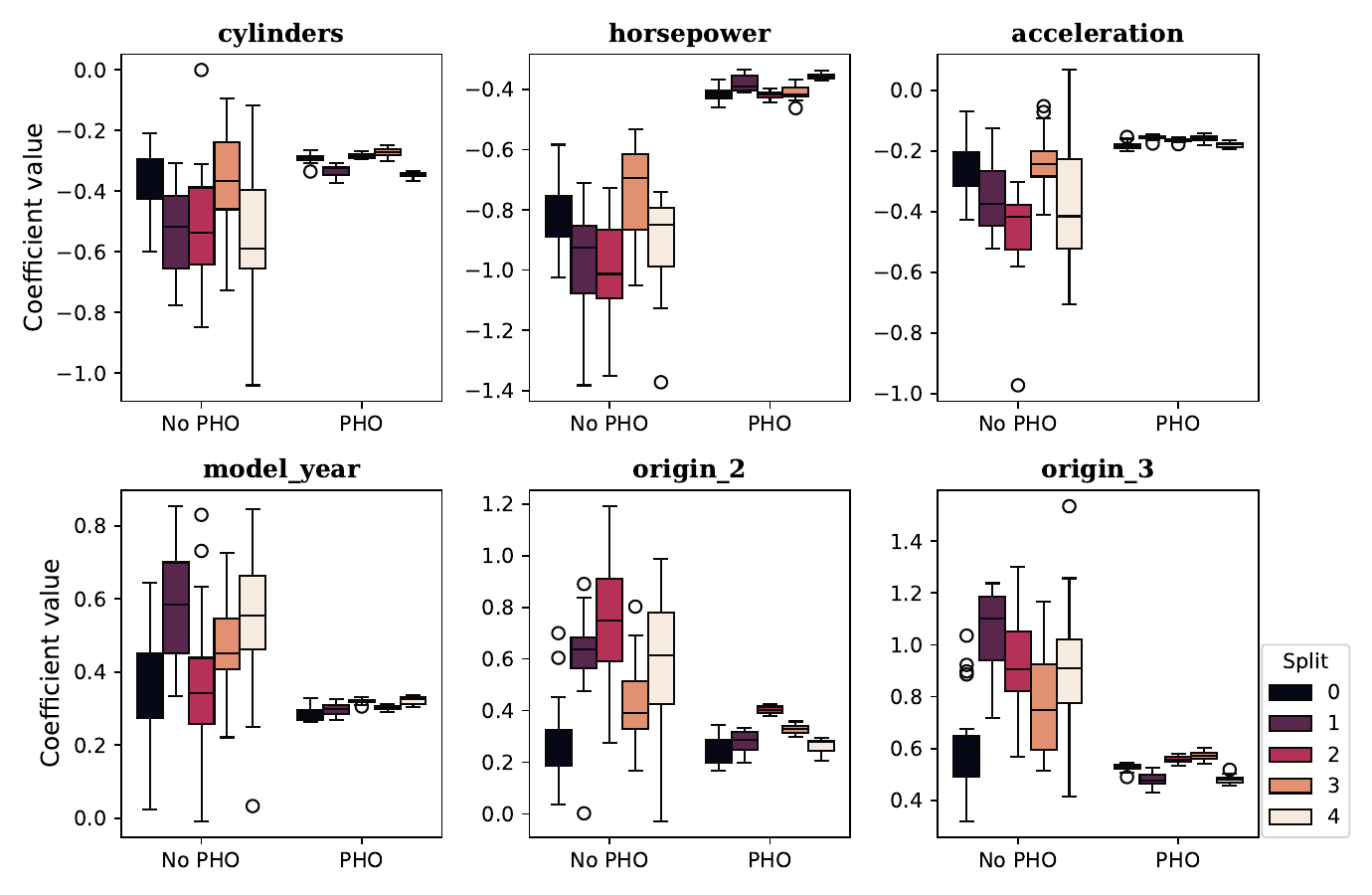}
	\caption{Fluctuation in LiD-GLM estimation of $\beta$ on all splits on the cars dataset with $T_d$ with and without PHO.}
	\label{fig: appl-cars_perf-PHO_vis_with_nu_d}
\end{figure}

\begin{figure}[t]
    \includegraphics[width=\textwidth]{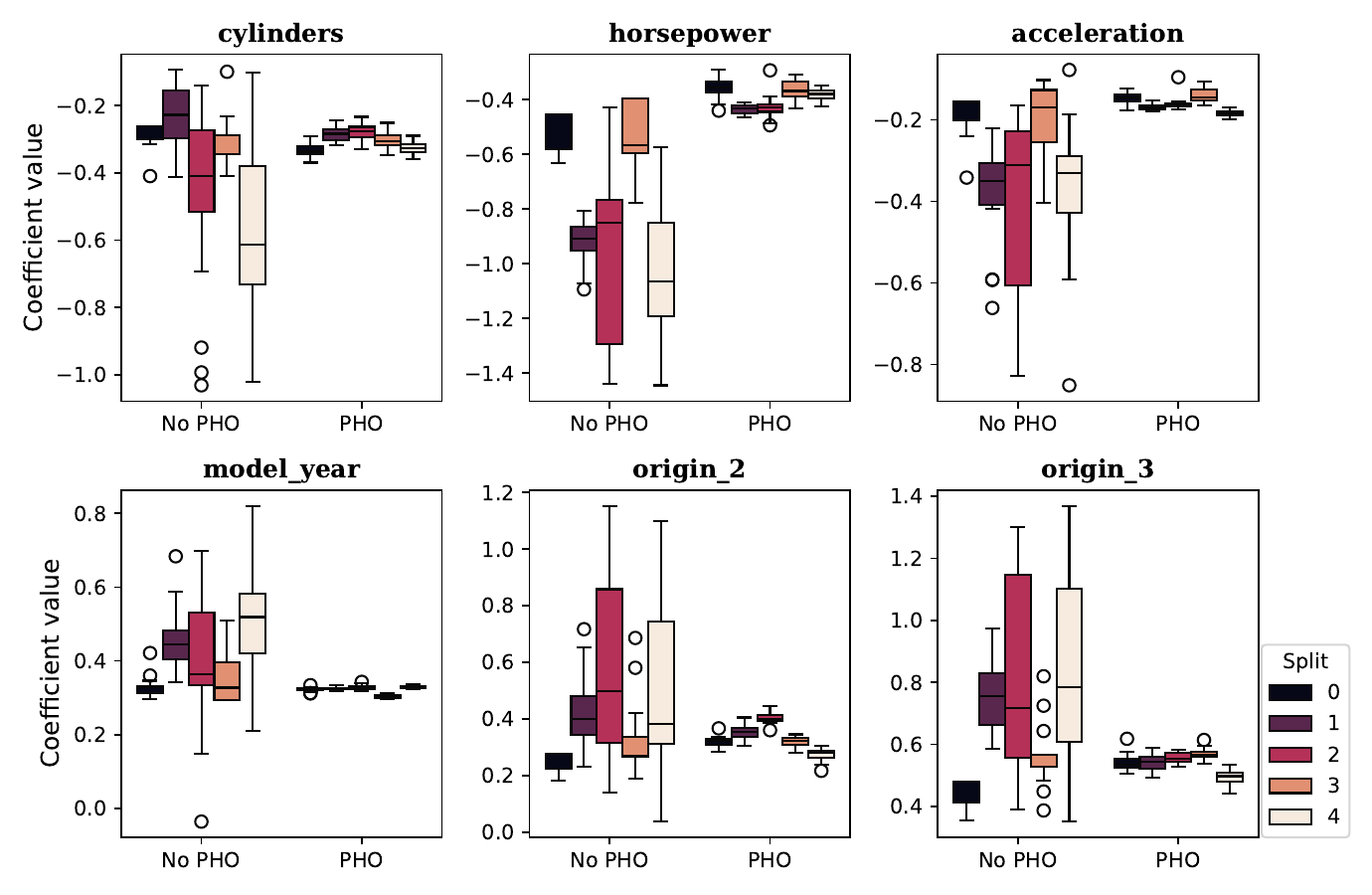}
	\caption{Fluctuation in LiD-GLM estimation of $\beta$ on all splits on the cars dataset without $T_d$ with and without PHO.}
	\label{fig: appl-cars_perf-PHO_vis_no_nu_d}
\end{figure}

\begin{figure}[t]
    \includegraphics[width=\textwidth]{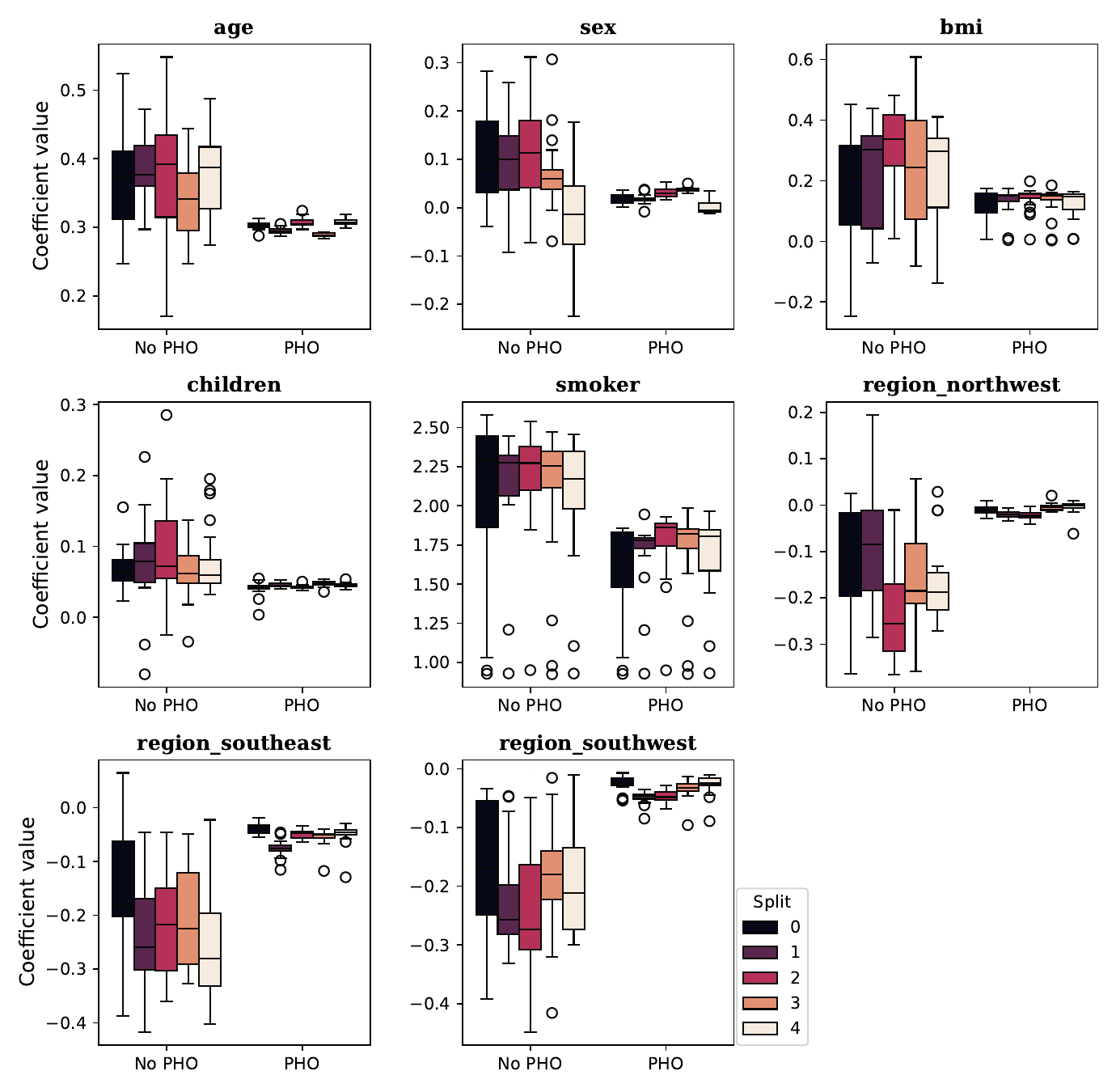}
	\caption{Fluctuation in LiD-GLM estimation of $\beta$ on all splits on the insurance dataset with $T_d$ with and without PHO.}
	\label{fig: appl-insurance_perf-PHO_vis_with_nu_d}
\end{figure}

\begin{figure}[t]
    \includegraphics[width=\textwidth]{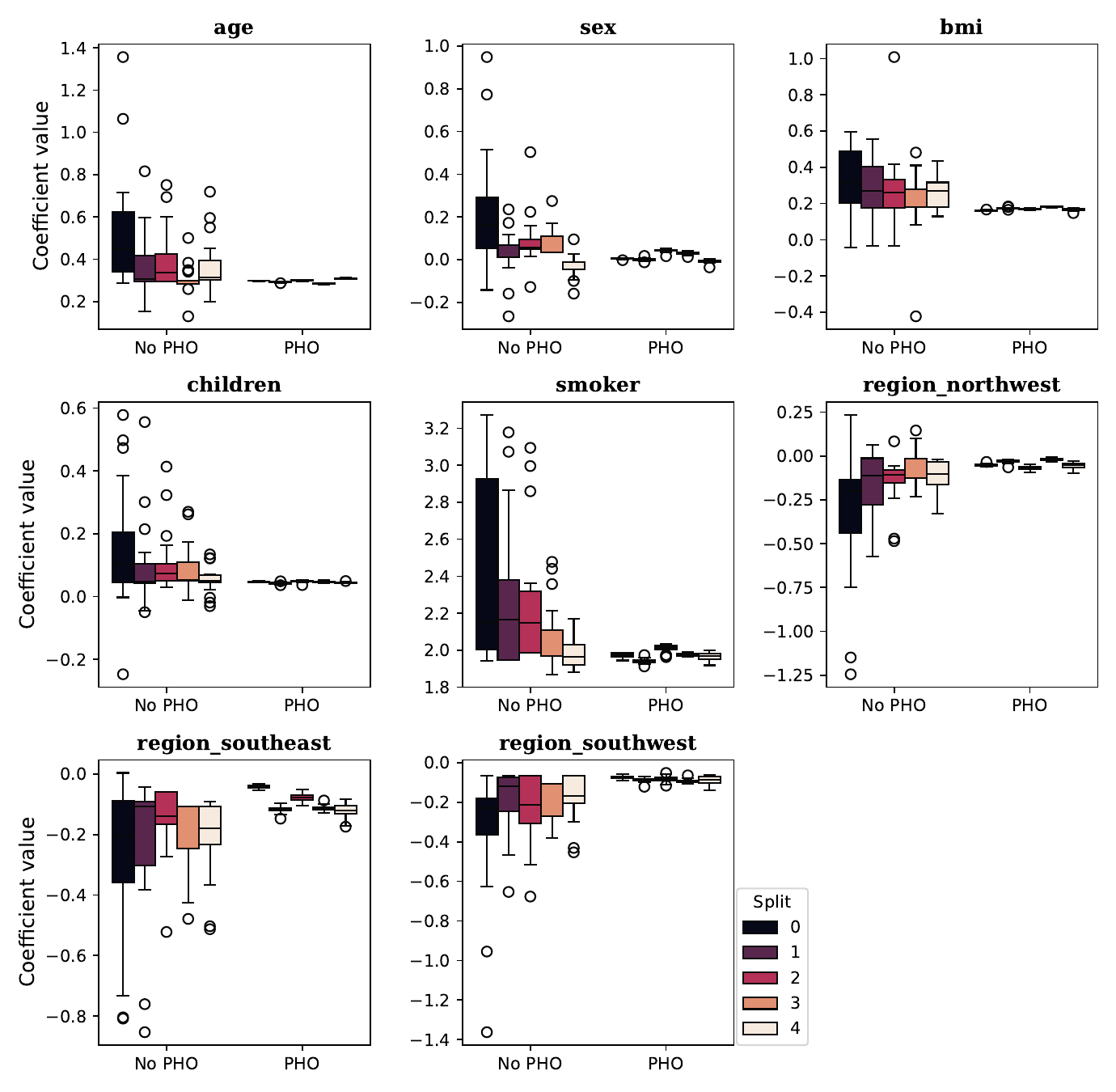}
	\caption{Fluctuation in LiD-GLM estimation of $\beta$ on all splits on the insurance dataset without $T_d$ with and without PHO.}
	\label{fig: appl-insurance_perf-PHO_vis_no_nu_d}
\end{figure}

\begin{figure}[t]
    \includegraphics[width=\textwidth]{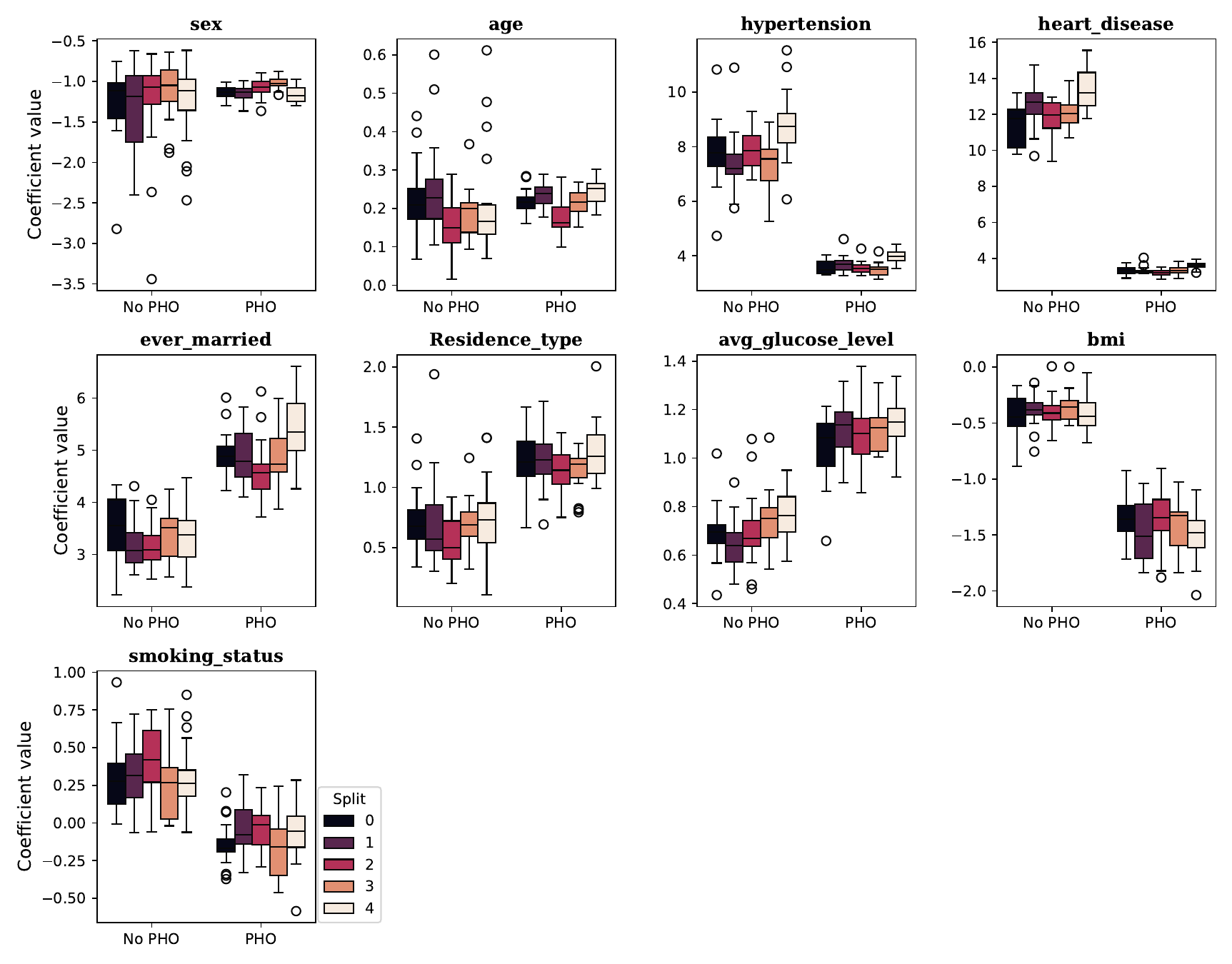}
	\caption{Fluctuation in LiD-GLM estimation of $\beta$ on all splits on the stroke dataset with and without PHO.}
	\label{fig: appl-stroke_perf-PHO_vis_all_splits}
\end{figure}